%% file: root.tex
\documentclass[letterpaper, 10pt, conference]{cssconf}
\usepackage[pdftex]{graphicx}
\usepackage{amsmath,amsfonts,amssymb}
\usepackage{mathtools}
\usepackage{times}
\usepackage{float}
\usepackage{multirow}
\usepackage{paralist, tabularx}
\usepackage{hyperref}
\usepackage{algorithm}
\usepackage{algpseudocode}
\usepackage[normalem]{ulem}
\usepackage{cite}
\usepackage{color}
\usepackage{booktabs}
\usepackage{makecell}
\usepackage{epsfig}
\usepackage{siunitx}
\usepackage{colortbl}
\usepackage[dvipsnames]{xcolor}
\usepackage{soul}

\newtheorem{theorem}{Theorem}
\newtheorem{remark}{Remark}
\newtheorem{lemma}{Lemma}
\newtheorem{proposition}{Proposition}

\newtheorem{assumption}{Assumption}

\newcommand{\R}{\mathbb{R}}
\newcommand{\norm}[1]{\left\|#1\right\|}

\newcommand{\diag}{\operatorname{diag}}
\DeclareMathOperator*{\argmin}{arg\,min}

\providecommand{\IEEEoverridecommandlockouts}{}
\providecommand{\overrideIEEEmargins}{}
\providecommand{\QED}{\hfill\blacksquare}

\makeatletter
\@ifundefined{iflatexml}{%
  \newif\iflatexml
  \latexmlfalse
}{}
\makeatother

\IEEEoverridecommandlockouts 
\title{\LARGE \bf
Safety-Critical Adaptive Impedance Control via Nonsmooth Control Barrier Functions under State and Input Constraints
}

\author{
Faisal Lawan, Xiaoran Han, Joaquin Carrasco, Barry Lennox, Xiaoxiao Cheng
\thanks{This work was supported by the RAICo Programme and the University of Manchester Strategic Investment Reserve Fund.}
\thanks{F. Lawan, X. Han, J. Carrasco, B. Lennox, and X. Cheng are with the Department of Electrical and Electronic Engineering, The University of Manchester, M13 9PL, United Kingdom, (Corresponding authors: F. Lawan, \texttt{faisal.lawan@manchester.ac.uk}; X. Cheng, \texttt{xiaoxiao.cheng@manchester.ac.uk}.)}
}

\begin{document}
\maketitle
\thispagestyle{empty}
\pagestyle{empty}

\begin{abstract}
    Safe physical interaction is critical for deploying robotic manipulators in human–robot interaction and contact-rich tasks, where uncertainty, external forces, and actuator limitations can compromise both performance and safety. We propose an online adaptive impedance control framework that enforces joint-state safety while achieving compliant interaction under uncertain dynamics. The approach combines a quadratic-program-based safety filter with a novel composed position–velocity non-smooth control barrier function, enabling joint position and velocity constraints to be enforced through a unified relative-degree-one barrier. Unknown dynamics are compensated online using an interval type-2 fuzzy logic system, while actuator torque limits are handled through soft constraints with exact penalty recovery of feasible solutions. A disturbance-observer-enhanced safety mechanism improves robustness against modelling errors and external interaction forces. Using composite Lyapunov analysis, we prove forward invariance of the safe set and the uniform ultimately boundedness of the impedance-tracking error. Simulations on a 7-DOF manipulator with severe parametric uncertainty and external interaction wrenches demonstrate safe constraint satisfaction and robust impedance tracking.
\end{abstract}

%%%%%%%%%%satisfaction of torque constraints %%%%%%%%%%%%%%%%%%%%%%%%%%%%%%%%%%%%%%%
\input{section/Introduction}
\input{section/Problem}
\input{section/Methodology}
\input{section/Theoretical-Results}
\input{section/Tuning-Guidelines}
\input{section/Simulation}
\input{section/Conclusion}
\bibliography{ref}
\bibliographystyle{IEEEtran}
\input{section/Appendix}

\end{document}

%% file: section/Introduction.tex
\section{Introduction}
    \par Physical human–robot interaction (HRI) requires robots to operate safely in proximity to humans while interacting with uncertain and dynamic environments. Impedance control is therefore critical, as it enables robots to regulate the dynamic relationship between motion and interaction forces, allowing safe and compliant contact with humans and uncertain environments~\cite{hogan1985impedance, spong_robot_2022}. It becomes particularly important in tasks such as surface interaction~\cite{li2018force}, haptic sensing~\cite{uttayopas2023object}, and operator assistance in remote handling~\cite{noccaro2025robot}, where accurate dynamic models are unavailable and stable contact must be maintained under uncertainty.

    \par However, impedance control in dynamic HRI faces several key challenges in extreme environments such as nuclear decommissioning. First, robot manipulator dynamics are subject to significant uncertainties, necessitating adaptive regulation during interaction to accurately track the desired impedance behaviour~\cite{xu_adaptive_2023}. Second, robotic joints are subject to safety constraints, including limits on joint position, velocity, and actuator torque, which must be explicitly incorporated into the control design~\cite{zhang2004unified}. Third, discrepancies in task perception and motion planning between the human and the robot can lead to dynamically inconsistent interaction phases~\cite{cheng2025human}, resulting in uncertain and potentially destabilising contact forces that further complicate safety enforcement~\cite{haddadin2016physical}. Incorporating both safety requirements and dynamic uncertainties into impedance control design for robotic manipulators remains a significant challenge.

    \par Among model-based approaches to impedance control, computed torque control~\cite{spong_robot_2022} and Model Predictive Control (MPC) are widely used. While MPC can explicitly handle constraints \cite{bednarczyk2020model}, its effectiveness relies on accurate knowledge of system dynamics, which is often unavailable in HRI applications due to varying human-induced loads and model mismatch. Adaptive control methods address this limitation by estimating uncertain dynamics and updating control parameters online. Classical adaptive schemes typically rely on the linear-in-the-parameters (LIP) assumption~\cite{spong_robot_2022}, whereas learning-based approaches, such as neural networks (NNs) and fuzzy logic systems (FLSs), relax this requirement~\cite{tang_fuzzy_2024}. However, neither class of methods guarantees the satisfaction of state or input constraints.

    \par Control Barrier Function (CBF)~\cite{wabersich_data-driven_2023} provides an effective approach for embedding safety constraints directly into control design by ensuring forward invariance of a safe set. Although standard CBF approaches also require accurate knowledge of the dynamics, several extensions address model uncertainty by using worst-case bounds, disturbance observers, and data-driven estimates~\cite{wabersich_data-driven_2023, sweatland_adaptive_2025}. In particular, the Adaptive Deep Neural Network (ADNN)-based CBF~\cite{sweatland_adaptive_2025} enforces safety online under bounded uncertainty, but it is computationally intensive and limited to tasks with sufficient excitation. The framework in~\cite{dey_state-constrained_2025} combines NNs with CBFs to enable safe online adaptive tracking of manipulators. However, it is restricted to tracking tasks, can induce chattering, and does not explicitly account for actuator torque constraints.

    \par To address the problem of safe impedance tracking in human–robot interaction under multiple constraints, we propose an online adaptive controller that tracks a desired impedance model during human interaction while enforcing state safety and respecting actuator torque limits whenever hard feasibility holds, without requiring prior offline training data. An Interval Type-2 Fuzzy Logic System (IT2-FLS)~\cite{ying_fuzzy_2000, mohammadzadeh_modern_2023} is employed as the function approximator, chosen for its interpretable rule structure and transparent compensation behaviour.

    \par The proposed method operates as follows. A unified soft-constrained quadratic program (QP) with an exact penalty function~\cite{kerrigan_soft_2000} simultaneously enforces joint state constraints, formulated via novel non-smooth CBFs (NCBFs)~\cite{glotfelter_nonsmooth_2017}, as hard constraints while softening actuator torque constraints. A DOB-based robust modification of the NCBFs~\cite{wang_disturbance_2023} robustifies the safety filter under model uncertainty. The QP output is then integrated into a command-driven reference model, adopted from model reference adaptive control (MRAC)~\cite{narendra_stable_2012}, whose filtered reference-model tracking error drives the IT2-FLS adaptation online. Because the reference model is driven by the constrained QP output rather than the unconstrained nominal command, the filtered error $\mathbf{r}$ reflects reference-model tracking mismatch rather than constraint enforcement, thereby decoupling the adaptation loop from the safety filter. By softening only the torque constraints with a sufficiently large penalty weight, the QP solution coincides with the hard-constrained solution whenever feasible and gracefully relaxes torque limits otherwise, thereby eliminating chattering due to infeasibility.

    \par The \textit{primary contributions} of this work are as follows:
    \begin{enumerate}
        \item A novel position-velocity NCBF is composed that encodes joint position and velocity limits in a single relative-degree-one barrier function, yielding a Lipschitz-continuous safety filter without requiring higher-order CBF constructions. This NCBF is embedded as a hard constraint within a unified soft-constrained QP, together with a DOB-based robust modification~\cite{wang_disturbance_2023}, thereby enforcing joint position and velocity safety invariantly while accommodating dynamic uncertainties.
        \item An IT2-FLS-based online adaptive controller that learns unknown structured dynamics in real time. A command-driven reference model, in the spirit of MRAC~\cite{narendra_stable_2012}, integrates the QP output so that the filtered reference-model error $\mathbf{r}$ reflects reference-model tracking mismatch only, ensuring constraint-aware learning. Exact penalty functions~\cite{kerrigan_soft_2000} recover hard torque feasibility when the interaction wrench is admissible and avoid chattering under infeasibility.
        \item A composite Lyapunov stability analysis proving forward invariance of the NCBF safe set $\mathcal{C}^* \subseteq \mathcal{C}_0$, torque satisfaction on the admissible wrench set $\mathcal{F}(\mathbf{x})$, and Uniformly Ultimately Bounded (UUB) impedance tracking with tunable bounds.
    \end{enumerate}

    \par \textbf{Notations:} Let $\R_{\geq0} \coloneqq [0,\infty)$, $\R_{>0} \coloneqq (0,\infty)$, and $\R^{n\times m}$ represent the space of $n\times m$ matrices. Boldface variables denote vectors or matrices. The identity matrix of size $n$ is denoted by $\mathbf{I}_n$, and $\mathbf{0}$ is a sufficiently-size matrix of zeros. Let $\norm{\cdot}$ be the Euclidean norm for vectors and the induced 2-norm for matrices. Let $[d] \coloneqq \{1, \dots, d\}$, for any $d > 1$. For $\mathbf{x} \in \R^n,\ \mathbf{y} \in \R^m$, $(\mathbf{x}, \mathbf{y}) \coloneqq [\mathbf{x}^\top, \mathbf{y}^\top]^\top$. For vectors $\mathbf{u}, \mathbf{v} \in \mathbb{R}^n$, an operation $\circ$ applied \emph{element-wise} is defined by $(\mathbf{u} \circ \mathbf{v})_i = u_i \circ v_i$ for all $i \in [n]$. For a modelling variable $X$, its nominal value is $\hat{X}$, and $\tilde{X} = \hat{X} - X$. The Kronecker product is denoted by $\otimes$; let $|\mathbf{x}|$ denote the element-wise absolute value. The inner product of vectors $\mathbf{a}, \mathbf{b} \in \mathbb{R}^n$ is represented by $\langle \mathbf{a}, \mathbf{b} \rangle$. For a continuously differentiable convex function $f: \R^n \to \R$ and a symmetric positive-definite matrix $\boldsymbol{\Gamma} \in \R^{n \times n}$, the smooth projection operator is defined as $\operatorname{Proj}_{\boldsymbol{\Gamma}}(\boldsymbol{\theta}, \mathbf{y}, f) \coloneqq \boldsymbol{\Gamma} \mathbf{y} - \boldsymbol{\Gamma} \frac{\nabla f(\boldsymbol{\theta})(\nabla f(\boldsymbol{\theta}))^\top}{(\nabla f(\boldsymbol{\theta}))^\top \boldsymbol{\Gamma} \nabla f(\boldsymbol{\theta})} \boldsymbol{\Gamma} \mathbf{y} f(\boldsymbol{\theta})$ if $f(\boldsymbol{\theta}) > 0 \land \mathbf{y}^\top \boldsymbol{\Gamma} \nabla f(\boldsymbol{\theta}) > 0$, and $\operatorname{Proj}_{\boldsymbol{\Gamma}}(\boldsymbol{\theta}, \mathbf{y}, f) \coloneqq \boldsymbol{\Gamma} \mathbf{y}$ otherwise.

%% file: section/Problem.tex
\section{Problem Formulation}\label{sec:problem}
    \par Following \cite{spong_robot_2022}, consider a rigid $n$-DOF robotic manipulator with Euler-Lagrange dynamics 
    \begin{equation}\label{eq: Robot Model}
        \mathbf{M}(\mathbf{q})\ddot{\mathbf{q}} + \mathbf{C}(\mathbf{q}, \dot{\mathbf{q}})\dot{\mathbf{q}} + \mathbf{G}(\mathbf{q}) + \boldsymbol{\tau}_d(t) = \mathbf{u} + \boldsymbol{\tau}_h(t),
    \end{equation}
    where $\mathbf{q}, \dot{\mathbf{q}}, \ddot{\mathbf{q}} \in \R^n $ represent the robot's joint position, velocity, and acceleration vectors; $\mathbf{M}: \R^n \to \R^{n\times n}$ is the positive-definite inertia matrix; $\mathbf{C}: \R^n \times \R^n \to \R^{n\times n}$ is the Coriolis matrix; $\mathbf{G}: \R^n \to \R^n$ is the gravity torque vector; $\boldsymbol{\tau}_d \in \R^n$ is an unknown disturbance term (arising from unmodelled dynamics, friction, etc.); $\boldsymbol{\tau}_h \coloneqq \mathbf{J}(\mathbf{q})^\top \mathbf{F}_h(t)$ is the human interaction torque, where $\mathbf{F}_h(t) \in \R^6$ is the end-effector wrench applied by the human and $\mathbf{J}(\mathbf{q}) \in \R^{6\times n}$ is the manipulator's Jacobian; $\mathbf{u} \in \R^n$ is the actuator torque. The wrench $\mathbf{F}_h(t)$ is assumed measurable (e.g., via a wrist force--torque sensor).

    \begin{assumption}\label{assum: Scenario}
        The following assumptions regarding the robotic manipulator and its task are made:
        \begin{enumerate}
            \renewcommand{\labelenumi}{(\roman{enumi})}
            \item a bounded, continuously differentiable trajectory $(\mathbf{z}_d(t), \dot{\mathbf{z}}_d(t), \ddot{\mathbf{z}}_d(t))$ in task space is predefined, where $\mathbf{z}_d(t) \in \R^6$;
            \item the smooth forward kinematics denoted $\mathbf{h}_z: \R^n \to \R^6$ and Jacobian denoted $\mathbf{J}: \R^n \to \R^{6\times n}$ are known so that the actual end-effector pose satisfies $\mathbf{z} = \mathbf{h}_z(\mathbf{q})$;
            \item the disturbance  $\boldsymbol{\tau}_d(t)$ is bounded and Lipschitz-continuous.
        \end{enumerate}
    \end{assumption}

    \par Assume that the true parameters are unknown, but identified nominal rigid-body parameters $(\hat{\mathbf{M}}(\mathbf{q}), \hat{\mathbf{C}}(\mathbf{q}, \dot{\mathbf{q}}), \hat{\mathbf{G}}(\mathbf{q}))$ are available \cite{gaz_dynamic_2019}, giving the modelled dynamics
    \begin{equation}\label{eq: Estimated Model}
        \hat{\mathbf{M}}(\mathbf{q})\ddot{\mathbf{q}} + \hat{\mathbf{C}}(\mathbf{q}, \dot{\mathbf{q}})\dot{\mathbf{q}} + \hat{\mathbf{G}}(\mathbf{q}) = \mathbf{u} + \boldsymbol{\tau}_h(t).
    \end{equation}

    \par Consider the desired impedance model dynamics below
    \begin{equation}\label{eq: Impedance Model}
        \mathbf{M}_d \ddot{\mathbf{e}}_{z, \mathrm{imp}}(t) + \mathbf{B}_d \dot{\mathbf{e}}_{z, \mathrm{imp}}(t) + \mathbf{K}_d \mathbf{e}_{z, \mathrm{imp}}(t) =  \mathbf{F}_h (t),
    \end{equation}
    where $\mathbf{e}_{z, \mathrm{imp}}(t) = \mathbf{z}_{d}(t) - \mathbf{z}_{\mathrm{imp}}(t)$ is a virtual task-space tracking error, $\mathbf{M}_d,\ \mathbf{B}_d,\ \mathbf{K}_d \in \R^{6\times 6}$ are positive-definite designer matrices, and $(\mathbf{z}_{\mathrm{imp}}, \dot{\mathbf{z}}_{\mathrm{imp}})$ solves \eqref{eq: Impedance Model}. In general, $\mathbf{z}_{\mathrm{imp}} \neq \mathbf{z}_d$ unless $\mathbf{F}_h = \mathbf{0}$. For notational simplicity, the independent variable $t$ is dropped.

    \par For a generally defined control input $\mathbf{u}$, let $\mathbf{x} = (\mathbf{q},\dot{\mathbf{q}})$ denote the full joint state and $\mathbf{e}_{z} = \mathbf{z}_{d} - \mathbf{z}$ the actual task-space tracking error. Let $\mathcal{C} \subset \R^{2n}$ denote the set of joint states that satisfy the prescribed position and velocity limits. For control design, we select a compact convex subset $\mathcal{C}_0 \subset \mathcal{C}$ that defines the intended safe operating region. Define the admissible torque set $\mathcal{U} \coloneqq [ \boldsymbol{\tau}_{\min}, \boldsymbol{\tau}_{\max}]$ and the barrier-free set $\mathcal{X}_b \subset \mathcal{C}_0$ as a compact set strictly in the interior of $\mathcal{C}_0$, where the CBF constraints are inactive. Let $\mathbf{q}_{\mathrm{imp}}$ denote the joint-space counterpart of $\mathbf{z}_{\mathrm{imp}}$ obtained by kinematic inversion \cite{spong_robot_2022}, and let $\mathbf{x}_{\mathrm{imp}} =(\mathbf{q}_{\mathrm{imp}},\dot{\mathbf{q}}_{\mathrm{imp}})$.

    \par Let $\mathcal{F}(\mathbf{x}) \subset \R^6$ denote the set of interaction wrenches for which hard satisfaction of both the state-safety and torque constraints is feasible at $\mathbf{x}$. A constructive characterization of $\mathcal{F}(\mathbf{x})$ is given in Section~\ref{subsec: pch}.
    
    \par The control objective is to design a Lipschitz-continuous control law $\mathbf{u}(\mathbf{x}, t)$ such that, for all $t \in \R_{\geq 0}$, the following statements hold:
    \begin{enumerate}
        \renewcommand{\labelenumi}{(\roman{enumi})}
        \item for every initial condition $\mathbf{x}(0) \in \mathcal{C}_0$, the closed-loop system satisfies $\mathbf{x}(t) \in \mathcal{C}_0$;
        \item whenever $\mathbf{F}_h(t) \in \mathcal{F}(\mathbf{x}(t))$, the actuator torque satisfies $\mathbf{u}(t) \in \mathcal{U}$;
        \item for all trajectories starting in $\mathcal{C}_0$, the tracking error of the desired impedance trajectory $\mathbf{z}_{\mathrm{imp}}(t) - \mathbf{z}(t)$ is uniformly ultimately bounded.
    \end{enumerate}

%% file: section/Methodology.tex
    \begin{figure*}[t]
      \centering
      \iflatexml
        % LaTeXML will process this line for the HTML build
        \includegraphics[width=\textwidth]{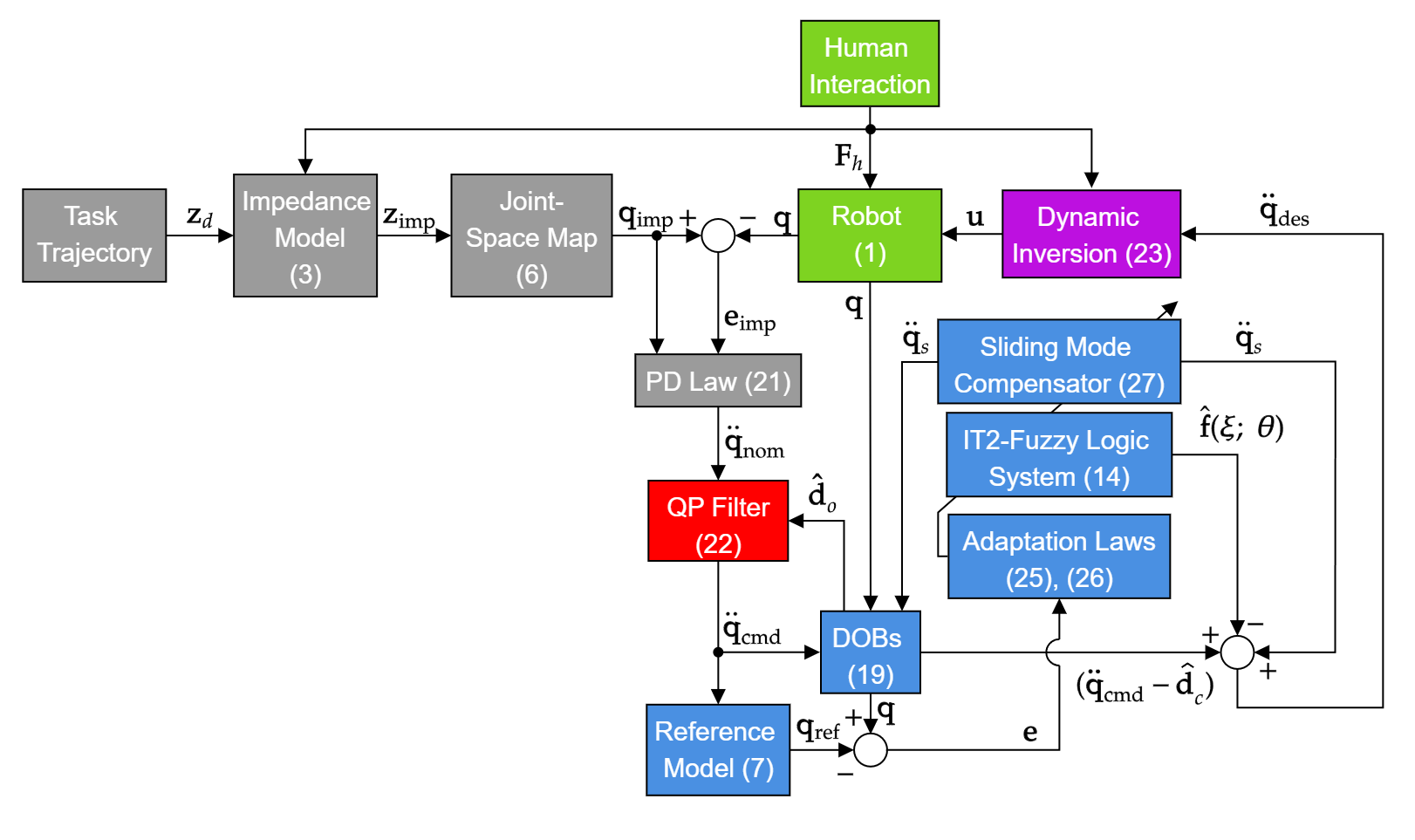}
      \else
        % pdflatex will process this line for the PDF build
        \includegraphics[width=\textwidth]{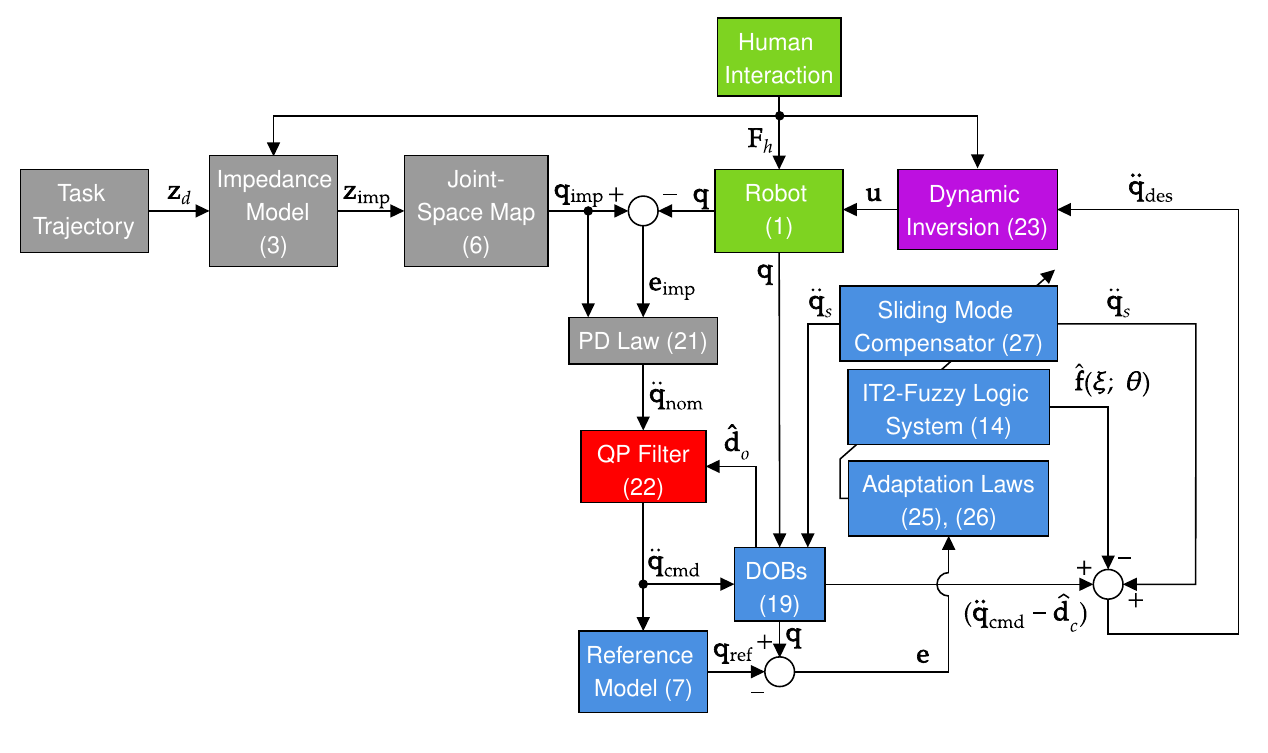}
      \fi
      \caption{Block diagram of the proposed safety-critical online adaptive impedance controller. The architecture is colour-coded by functional module: grey blocks handle impedance trajectory generation and nominal acceleration computation via the impedance model, joint-space map, and PD law; blue blocks provide online adaptation and constraint-aware reference tracking via the IT2-FLS, disturbance observers, sliding mode compensator, adaptation laws, and command-driven reference model; the red block enforces state and torque constraints through the unified soft-constrained QP with embedded NCBFs and exact penalty functions; green blocks represent the physical robot and human interaction; and the purple block executes dynamic inversion for torque generation.}
      \label{fig:architecture}
      \vspace{-5mm}
    \end{figure*}

\section{Methodology}\label{sec:design}
    \subsection{Impedance Trajectory Generation}\label{subsec: impedance}
        \par The impedance model~\eqref{eq: Impedance Model} is formulated in task space $\R^6$, whereas the controller operates in joint space $\R^n$. We map task-space impedance acceleration to joint space via DLS; alternative mappings may be used without affecting the subsequent guarantees.

        \par Let $\mathbf{J} = \mathbf{U} \boldsymbol{\Sigma} \mathbf{V}^\top$ be the singular value decomposition of the Jacobian. Because $n > 6$, near-singular configurations can arise, causing large joint velocities for small task-space commands. We employ an \emph{adaptive damped least-squares} (DLS) inverse~\cite{spong_robot_2022}:
        \begin{equation}\label{eq: DLS}
            \mathbf{J}^{\dagger}_{\mathrm{DLS}} = \mathbf{V} \boldsymbol{\Sigma} (\boldsymbol{\Sigma}^2 + \mathbf{K}_{\lambda})^{-1} \mathbf{U}^\top,
        \end{equation}
        where $\mathbf{K}_{\lambda} = \diag(k_{\lambda,1},\dots,k_{\lambda,6})$ with $k_{\lambda, i} = k_{\max}\, g(\sigma_i)$, $g(\sigma_i) = 2/(1 + e^{\alpha_s \sigma_i})$ is a sigmoid that scales the damping inversely with the singular value $\sigma_i$, and $k_{\max}, \alpha_s \in \R_{>0}$ are design parameters. The damping thus vanishes when the manipulator is far from a singularity and grows smoothly as any singular value decreases, bounding joint velocities without introducing discontinuities.

        \par However, near a singularity, DLS alone can still produce a task-space acceleration that demands motion in nearly unreachable directions. To prevent this, we apply a \emph{singularity-consistent} projection to the task-space impedance acceleration $\ddot{\mathbf{z}}_{\mathrm{imp}} \in \R^6$. Let $\mathbf{u}_i$ be the $i$-th column of $\mathbf{U}$ and define the weighting
        \begin{equation}\label{eq: singularity weight}
            w_i = \begin{cases}
                0, & \sigma_i \geq \bar{\sigma}, \\
                \tfrac{1}{2}\!\left(1 + \cos\!\left(\pi\tfrac{\sigma_i - \underline{\sigma}}{\bar{\sigma} - \underline{\sigma}}\right)\right), & \underline{\sigma} < \sigma_i < \bar{\sigma}, \\
                1, & \sigma_i \leq \underline{\sigma},
            \end{cases}
        \end{equation}
        where $\underline{\sigma} < \bar{\sigma}$ are activation thresholds. The singularity-consistent acceleration is then $\bar{\ddot{\mathbf{z}}}_{\mathrm{imp}} = \ddot{\mathbf{z}}_{\mathrm{imp}} - \sum_{i} w_i (\mathbf{u}_i^\top \ddot{\mathbf{z}}_{\mathrm{imp}}) \mathbf{u}_i$, which smoothly zeroes acceleration components in near-singular task-space directions, preventing discontinuous jumps in the joint-space reference.

        \par To utilise the redundancy of the manipulator for energy dissipation, we add a null-space damping term. The resulting impedance joint acceleration is
        \begin{equation}\label{eq: impedance joint accel}
            \begin{aligned}
                \ddot{\mathbf{q}}_{\mathrm{imp}} &= \mathbf{J}^{\dagger}_{\mathrm{DLS}}\!\left(\bar{\ddot{\mathbf{z}}}_{\mathrm{imp}} - \dot{\mathbf{J}}\dot{\mathbf{q}}_{\mathrm{imp}}\right) \\
                &- (\mathbf{I}_n - \mathbf{J}^{\dagger}_{\mathrm{DLS}} \mathbf{J})(\mathbf{K}_{\mathrm{damp}}\dot{\mathbf{q}}_{\mathrm{imp}}),
            \end{aligned}
        \end{equation}
        where $\mathbf{K}_{\mathrm{damp}} \in \R^{n\times n}$ is a positive-definite damping gain. The null-space projector $(\mathbf{I}_n - \mathbf{J}^{\dagger}_{\mathrm{DLS}} \mathbf{J})$ ensures that the joint damping does not interfere with the task-space impedance behaviour.

    \subsection{Safety-Filtered Reference Model and NCBF Design}\label{subsec: Ref Gov}
        \par The NCBFs below define hard acceleration bounds in the unified QP of Section~\ref{subsec: pch}. A command-driven reference model~\cite{narendra_stable_2012} tracks the QP output $\ddot{\mathbf{q}}_{\mathrm{cmd}}$ through $\mathbf{x}_{\mathrm{ref}} = (\mathbf{q}_{\mathrm{ref}},\dot{\mathbf{q}}_{\mathrm{ref}}) \in\R^{2n}$ with double-integrator dynamics
        \begin{equation}\label{eq: Reference Governor}
            \dot{\mathbf{x}}_{\mathrm{ref}} = \begin{bmatrix}
                \mathbf{0} & \mathbf{I}_n \\
                \mathbf{0} & \mathbf{0}
            \end{bmatrix}\mathbf{x}_{\mathrm{ref}} + \begin{bmatrix}
                \mathbf{0} \\
                \mathbf{I}_n
            \end{bmatrix}\ddot{\mathbf{q}}_{\mathrm{u}},
        \end{equation}
        where $\ddot{\mathbf{q}}_{\mathrm{u}} = \ddot{\mathbf{q}}_{\mathrm{cmd}} + \mathbf{K}_v(\dot{\mathbf{q}} - \dot{\mathbf{q}}_{\mathrm{ref}}) + \mathbf{K}_p(\mathbf{q} - \mathbf{q}_{\mathrm{ref}})$, $\mathbf{K}_v, \mathbf{K}_p \in \R^{n\times n}$ are positive-definite diagonal matrices satisfying $\mathbf{K}_v = \boldsymbol{\Lambda} + \mathbf{K}_r$ and $\mathbf{K}_p = \mathbf{K}_r \boldsymbol{\Lambda}$, and $(\boldsymbol{\Lambda}, \mathbf{K}_r)$ are defined in Section~\ref{subsec: fuzzy}.

        \par Let $\mathcal{J} \coloneqq [n]$ denote the set of joint indices. Let $\mathcal{J}_v$ denote joints subject only to velocity constraints, and $\mathcal{J}_p$ those subject to both position and velocity constraints ($\mathcal{J} = \mathcal{J}_v \cup \mathcal{J}_p$). The physical limit set $\mathcal{C} \coloneqq \{ \mathbf{x} \in \R^{2n} : (|\dot{q}_i| \leq  \dot{q}_{\max, i}\, \forall i \in \mathcal{J}) \; \land  (|q_i| \leq q_{\max, i}\, \forall i \in \mathcal{J}_p) \}$ was introduced in Section~\ref{sec:problem}; the designer selects a compact convex subset $\mathcal{C}_0 \subset \mathcal{C}$ as the intended operating region.

        \subsubsection*{Velocity-only joints ($i\in\mathcal{J}_v$)}
        \par A valid NCBF for the velocity constraint set is
        \begin{equation}\label{eq: velocity NCBF}
            {^i}h_v(\mathbf{x}) \coloneqq \min_{j \in [2]}{\{{^i}h_{vj}(\mathbf{x}) \}},
        \end{equation}
        where ${^i}h_{v1}(\mathbf{x}) \coloneqq \dot{q}_{\max, i} - x_2$ and ${^i}h_{v2}(\mathbf{x}) \coloneqq \dot{q}_{\max, i} + x_2$. Here, the left superscript $i$ denotes the barrier function for the $i$-th joint, and $x_{1},x_{2}$ are its position and velocity states. 

        \begin{remark}
            A function is \emph{Clarke regular} at a point if its generalised directional derivative coincides with the one-sided directional derivative~\cite{glotfelter_nonsmooth_2017}. The function \eqref{eq: velocity NCBF} is not Clarke regular due to its boolean composition~\cite{glotfelter_nonsmooth_2017}. As ${^i}h_v$ is Lipschitz, but not Clarke regular, we invoke the weak set-valued Lie derivative result~\cite{glotfelter_nonsmooth_2017} for the NCBF.
        \end{remark}

        \subsubsection*{Position-and-velocity joints ($i\in\mathcal{J}_p$)}
        \par Composing ${^i}h_{v j}$ naively with position barriers directly yields a relative-degree-2 constraint, making the CBF-Quadratic Programme (CBF-QP) solution non-Lipschitz~\cite{ames_control_2019, kim_is_2026}. Instead, we propose the NCBF composition:
        \begin{equation}\label{eq: position NCBF}
            {^i}h_p(\mathbf{x}) \coloneqq \min_{j \in [2]}{\{{^i}h_{pj}(\mathbf{x}) \}},
        \end{equation}
        with ${^i}h_{p1}(\mathbf{x}) \coloneqq 1 - R_1^l\!\left(\tfrac{x_1}{q_{\max, i}}\right) - R_1\!\left(\tfrac{x_2}{\dot{q}_{\max, i}} \right)$ and ${^i}h_{p2}(\mathbf{x}) \coloneqq 1 - R_2^l\!\left(\tfrac{x_1}{q_{\max, i}}\right) - R_2\!\left( \tfrac{x_2}{\dot{q}_{\max, i}} \right)$, where $R_1(\cdot) \coloneqq \max\{0, (\cdot) \}$, $R_2(\cdot) \coloneqq \max\{0, -(\cdot) \}$ and $l \geq 2$ is an even integer. Similar to ${^i}h_v$, ${^i}h_p$ is Lipschitz but not Clarke regular.

       \par  Applying the weak set-valued Lie derivative from~\cite{glotfelter_nonsmooth_2017} and Brezis' Theorem~\cite{brezis_characterization_1970}, the NCBF conditions 
        \begin{align}
                & \min{\mathcal{L}_F^W({^i}h_v(\mathbf{x}_i))} \geq -\gamma_v {^i}h_v(\mathbf{x}_i) , \quad \forall i \in \mathcal{J}_v, \label{eq: vel_cbf_cond}\\
                & \min{\mathcal{L}_F^W({^i}h_p(\mathbf{x}_i))} \geq -\gamma_p {^i}h_p(\mathbf{x}_i) , \quad \forall i \in \mathcal{J}_p, \label{eq: pos_cbf_cond}
        \end{align}
        where $\gamma_v, \gamma_p \in \R_{> 0}$ are tuning parameters controlling the convergence rate to the constraint boundary (larger $\gamma$ allows faster approach, but requires larger torque), and 
        \begin{align*}
            \mathcal{L}_F^W(h(\mathbf{x})) \coloneqq \{a \in \R &: \exists\ \mathbf{v} \in \mathbf{F}(\mathbf{x}),\ \boldsymbol{\zeta} \in \partial{h(\mathbf{x})} \\ & \text{ s.t.\ } \langle \boldsymbol{\zeta}, \mathbf{v} \rangle = a \},
        \end{align*}
        denotes the weak set-valued Lie derivative, with $\partial h$ the Clarke generalised gradient of $h$~\cite{glotfelter_nonsmooth_2017} and $\mathbf{F}(\mathbf{x})$ the flow of the closed-loop vector field, translate into the element-wise acceleration bounds $\boldsymbol{\alpha}_{A} \leq \ddot{\mathbf{q}} \leq \boldsymbol{\alpha}_{B}$. Let $\bar{l} = l -1$. For velocity-only joints $(\alpha_{A,i},\alpha_{B,i}) = (\alpha_{vA,i},\alpha_{vB,i})$ and for position-and-velocity joints $(\alpha_{A,i},\alpha_{B,i}) = (\alpha_{pA,i},\alpha_{pB,i})$, with
        \[
            \begin{aligned}
                \alpha_{vA, i} &= -\gamma_v{^i}h_{v2}(\mathbf{x}_{i}), \quad \alpha_{vB, i} = \gamma_v{^i}h_{v1}(\mathbf{x}_{i}), \\
                \alpha_{pA, i} &= \dot{q}_{\max, i}\!\left(-l\frac{\dot{q}_{i}}{q_{\max, i}}R_2^{\bar{l}} \!\left(\tfrac{q_{i}}{q_{\max, i}} \right) - \gamma_p {^i}h_{p2}(\mathbf{x}_{i}) \right), \\
                \alpha_{pB, i} &= \dot{q}_{\max, i}\!\left(-l\frac{\dot{q}_{i}}{q_{\max, i}} R_1^{\bar{l}} \!\left(\tfrac{q_{i}}{q_{\max, i}} \right) + \gamma_p {^i}h_{p1}(\mathbf{x}_{i}) \right).
            \end{aligned}
        \]
        These bounds are enforced as hard constraints in the QP~\eqref{eq: Torque Constraints}.

        \par Define the NCBF super-level set $\mathcal{C}^* \coloneqq \{\mathbf{x} \in \mathbb{R}^{2n} : {^i}h_v(\mathbf{x}_i) \geq 0,\, \forall i \in \mathcal{J}_v;\; {^i}h_p(\mathbf{x}_i) \geq 0,\, \forall i \in \mathcal{J}_p \}$, with $\mathcal{C}^* \subseteq \mathcal{C}_0 \subseteq \mathcal{C}$ for suitably chosen $(l, \gamma_v, \gamma_p, \beta, \bar{\omega}_1, \alpha_o, \nu)$. On $\mathcal{C}^*$, $\alpha_{A,i}\leq 0 \leq\alpha_{B,i}$ for $\gamma_v, \gamma_p > 0$.

    \subsection{Fuzzy Logic and Disturbance Observer Design}\label{subsec: fuzzy}
        \par Given $\mathbf{q}_{\mathrm{ref}}(t)$ from~\eqref{eq: Reference Governor}, define the reference-model tracking error $\mathbf{e}(t) = \mathbf{q}_{\mathrm{ref}}(t) - \mathbf{q}(t)$ and the filtered reference-model tracking error $\mathbf{r}(t) = \dot{\mathbf{e}}(t) + \boldsymbol{\Lambda} \mathbf{e}(t)$, where $\boldsymbol{\Lambda}, \mathbf{K}_r \in \R^{n\times n}$ are positive-definite diagonal matrices. Because $\mathbf{q}_{\mathrm{ref}}$ is driven by $\ddot{\mathbf{q}}_{\mathrm{cmd}}$, $\mathbf{r}$ reflects mismatch between the plant and the command-driven reference model, not constraint enforcement or impedance-model tracking error.

        \par Let the nominal control torque be
        \begin{equation}\label{eq: Nominal Control}
            \mathbf{u}_{\mathrm{nom}} 
            = \hat{\mathbf{M}}(\mathbf{q}) \ddot{\mathbf{q}}_{\mathrm{des}} 
            + \hat{\mathbf{C}}(\mathbf{q}, \dot{\mathbf{q}})\dot{\mathbf{q}} 
            + \hat{\mathbf{G}}(\mathbf{q}) 
            - \boldsymbol{\tau}_h,
        \end{equation}
        where $\ddot{\mathbf{q}}_{\mathrm{des}}$ is defined in \eqref{eq: Acceleration Control} below. Substituting \eqref{eq: Nominal Control} into \eqref{eq: Robot Model} gives the dynamics
        \begin{equation}\label{eq: Acceleration Dynamics}
            \ddot{\mathbf{q}} = \ddot{\mathbf{q}}_{\mathrm{des}} 
            + \mathbf{M}^{-1} \!\left(\tilde{\mathbf{M}} \ddot{\mathbf{q}}_{\mathrm{des}} 
            + \tilde{\mathbf{C}}\dot{\mathbf{q}} 
            + \tilde{\mathbf{G}} 
            - \boldsymbol{\tau}_d \right).
        \end{equation}

        \par Design $\ddot{\mathbf{q}}_{\mathrm{des}}$ to stabilise the tracking error in the presence of the unknown disturbance term and model mismatch:
        \begin{equation}\label{eq: Acceleration Control}
            \ddot{\mathbf{q}}_{\mathrm{des}} 
            = \ddot{\mathbf{q}}_{\mathrm{cmd}} 
            + \ddot{\mathbf{q}}_{s} 
            - \hat{\mathbf{d}}_c 
            - \hat{\mathbf{f}},
        \end{equation}
        where $\ddot{\mathbf{q}}_{\mathrm{cmd}}$ is the commanded acceleration from the QP~\eqref{eq: Torque Constraints}, $\hat{\mathbf{f}}$ \eqref{eq: fls} is an IT2-FLS approximator, $\hat{\mathbf{d}}_c$ \eqref{eq: dob} is a DOB-based compensator, and $\ddot{\mathbf{q}}_{s}$ \eqref{eq: Smooth SMC} is an adaptive sliding-mode compensator (SMC).

        \par Let $\bar{\boldsymbol{\xi}} = \ddot{\mathbf{q}}_{s} - \hat{\mathbf{d}}_c$. Substituting~\eqref{eq: Acceleration Control} into~\eqref{eq: Acceleration Dynamics}, one obtains
        \begin{equation}\label{eq: Simplified Dynamics}
            \ddot{\mathbf{q}} = \ddot{\mathbf{q}}_{\mathrm{des}} + \mathbf{D}(\boldsymbol{\xi}) - \boldsymbol{\delta}(t), \quad \forall t \in \R_{\geq 0},
        \end{equation}
        where $\mathbf{D}(\boldsymbol{\xi}) = \mathbf{M}^{-1} \!\left(\tilde{\mathbf{M}} (\ddot{\mathbf{q}}_{s} - \hat{\mathbf{d}}_c) + \tilde{\mathbf{C}}\dot{\mathbf{q}} + \tilde{\mathbf{G}} \right)$ is the structured uncertainty representing model mismatch, $\boldsymbol{\xi} = (\mathbf{q}, \dot{\mathbf{q}}, \bar{\boldsymbol{\xi}})$, and $\boldsymbol{\delta}(t)$ is the unstructured residual.

        \begin{assumption}\label{assum: compact}
            Consider $\boldsymbol{\xi} \in \Omega$, where $\Omega \subset \R^{3n}$ is compact. Let $\boldsymbol{\xi}_i \coloneqq \mathbf{S}_i \boldsymbol{\xi}$ with $\mathbf{S}_i \in \R^{3\times 3n}$ the selection matrix for the $i$-th joint, and $\Omega_i \coloneqq \mathbf{S}_i(\Omega) \subset \R^3$. There exist $n$ maps $f_i : \Omega_i \to \R$ and a residual $\mathbf{g} : \Omega \to \R^n$
            such that 
            \[
                \mathbf{D}(\boldsymbol{\xi}) = \mathbf{f}(\boldsymbol{\xi}) + \mathbf{g}(\boldsymbol{\xi}),
            \]
            where $\mathbf{f}(\boldsymbol{\xi}) \coloneqq (f_1, \dots, f_n)$.
        \end{assumption}

        \par Under Assumption~\ref{assum: compact}, by the universal approximation theorem \cite{wang_fuzzy_1992}, $\mathbf{f}$ can be approximated by a set of $n$ IT2-TSK-FLSs~\cite{ying_fuzzy_2000, mohammadzadeh_modern_2023}
        \begin{equation}\label{eq: fls}
            \hat{\mathbf{f}}(\boldsymbol{\xi}; \boldsymbol{\theta}) = \boldsymbol{\zeta}(\boldsymbol{\xi})^\top \boldsymbol{\theta},
        \end{equation}
        where $\boldsymbol{\zeta}(\boldsymbol{\xi}) \in \R^{N\times n}$ is the regressor containing normalised firing strengths from the type-reduced FLS output, $\boldsymbol{\theta} \in \R^N$ the consequent weight vector, $N = m \cdot n \cdot M$ with $m$ consequent parameters per rule, and $M$ rules. The IT2-FLS uses Gaussian interval-type-2 membership functions; for computational tractability, type-reduction is performed using the Nie-Tan operator \cite{mohammadzadeh_modern_2023}. The FLS inputs are chosen as $\boldsymbol{\xi}_i = (\sin(q_i), \cos (q_i),\dot{q}_i,\bar{\xi}_i)$, where $\bar{\xi}_i := \ddot{q}_{s,i} - \hat{d}_{c,i}$. Using $\sin(q_i), \cos (q_i)$ rather than $q_i$ directly ensures all fuzzy inputs remain Lipschitz, guaranteeing the compactness of $\Omega_i$.
        
        \par Define the ideal weights 
        \[
            \boldsymbol{\theta}^* = \arg \min_{\boldsymbol{\theta}} \sup_{\boldsymbol{\xi} \in \Omega} \norm{\mathbf{f}(\boldsymbol{\xi}) - \hat{\mathbf{f}}(\boldsymbol{\xi}; \boldsymbol{\theta})}^2,
        \]
        and the disturbance term
        \begin{equation}\label{eq: d_def}
            \mathbf{d}(t) = \mathbf{f}(\boldsymbol{\xi}) - \hat{\mathbf{f}}(\boldsymbol{\xi}; \boldsymbol{\theta}^*) - \bar{\boldsymbol{\delta}}(t), \quad \forall t \in \R_{\geq 0},
        \end{equation}
        where $\bar{\boldsymbol{\delta}}(t) \coloneqq \boldsymbol{\delta}(t) - \mathbf{g}(\boldsymbol{\xi})$ is a modified residual. The DOB-CBF disturbance is defined as
        \begin{equation}\label{eq: dcbf}
            \mathbf{d}_{\mathrm{CBF}}(t) = \mathbf{f}(\boldsymbol{\xi}) - \hat{\mathbf{f}}(\boldsymbol{\xi}; \boldsymbol{\theta}) - \bar{\boldsymbol{\delta}}(t) - \hat{\mathbf{d}}_c.
        \end{equation}

        \begin{assumption}\label{assum: bounds}
            For any $t \in \R_{\geq 0}$, the disturbances $\mathbf{d}(t), \mathbf{d}_{\mathrm{CBF}}(t)$ and their derivatives $\dot{\mathbf{d}}(t), \dot{\mathbf{d}}_{\mathrm{CBF}}(t)$ are bounded with known boundaries: $\|\mathbf{d}(t) \| \leq \omega_0$, $\|\mathbf{d}_{\mathrm{CBF}}(t) \| \leq \bar{\omega}_0$, $\|\dot{\mathbf{d}}(t) \| \leq \omega_1$, and $\|\dot{\mathbf{d}}_{\mathrm{CBF}}(t) \| \leq \bar{\omega}_1,\ \forall t \in \R_{>0}$, where $\omega_0, \bar{\omega}_0, \omega_1, \bar{\omega}_1 \in \R_{>0}$. In addition, the ideal weights are bounded by a known positive constant: $\|\boldsymbol{\theta}^* \| \leq \Theta$, where $\Theta \in \R_{>0}$.
        \end{assumption}

        \begin{remark}
            The existence of these bounds follows from the compactness of $\Omega$ (Assumption~\ref{assum: compact}), the Lipschitz continuity of $\mathbf{D}(\boldsymbol{\xi})$ \eqref{eq: Simplified Dynamics}, and standard DOB theory~\cite{wang_disturbance_2023}. Their exact values need not be known; instead, conservative bounds suffice.
        \end{remark}

        \par Under Assumption~\ref{assum: bounds} and the dynamics in \eqref{eq: Simplified Dynamics}, we adopt the DOBs of \cite{wang_disturbance_2023}, with the form
        \begin{equation}\label{eq: dob}
            \hat{\mathbf{d}}_c = \boldsymbol{\ell}_c + \alpha_c \dot{\mathbf{q}}, \quad \hat{\mathbf{d}}_o = \boldsymbol{\ell}_o + \alpha_o \dot{\mathbf{q}},
        \end{equation}
        where $\alpha_c, \alpha_o \in \R_{>0}$ are design parameters. The observer states evolve according to
        \begin{equation}\label{eq: dob_update}
            \dot{\boldsymbol{\ell}}_c = -\alpha_c (\ddot{\mathbf{q}}_{\mathrm{cmd}} + \ddot{\mathbf{q}}_{s}), \quad
            \dot{\boldsymbol{\ell}}_o = -\alpha_o (\ddot{\mathbf{q}}_{\mathrm{cmd}} + \ddot{\mathbf{q}}_{s} + \hat{\mathbf{d}}_o).
        \end{equation}
        
        \par $\hat{\mathbf{d}}_c$ estimates the residual disturbance after the learning loop, while $\hat{\mathbf{d}}_o$ estimates the total uncertainty in the system (used for the CBF safety analysis). Define the estimation error of the compensation DOB as $\mathbf{e}_d := \hat{\mathbf{d}}_c - \mathbf{d}(t)$. Under Assumption~\ref{assum: bounds} and following the analysis in \cite[Sec.~II-B]{wang_disturbance_2023}, $\mathbf{e}_d$ is uniformly ultimately bounded. We define this ultimate bound as follows: $\mathbf{w} = \sup|\mathbf{e}_d(t)|$, with $\norm{\mathbf{w}}\leq W$ for some $W \in \R_{>0}$.

    \subsection{Enforcing Constraints via Soft-Constrained Optimisation}\label{subsec: pch}
        \par The NCBF acceleration bounds from Section~\ref{subsec: Ref Gov} guarantee state safety when enforced, but do not account for actuator torque limits $\mathcal{U}$. When both CBF and torque constraints are imposed simultaneously, the resulting hard-constrained QP can become infeasible, e.g.\ during large interaction wrenches or severe model mismatch. We resolve this by embedding both constraint types in a single unified QP, \emph{softening} only the torque constraints with an exact penalty function~\cite{kerrigan_soft_2000}, while keeping the CBF constraints hard.

        \par Assume $\mathbf{e}_d(0) = \mathbf{0}$ (achievable by initialising~\eqref{eq: dob} offline). Let $\nu\in(0,2\alpha_o)$, $\max\{\gamma_v,\gamma_p\}\in (0,2\alpha_o-\nu)$, $\beta\in\R_{>0}$, and define the robust safety margin
        \begin{equation*}
          \ddot{q}_{\mathrm{marg},i} \coloneqq
          \begin{cases}
            \dfrac{\bar\omega_1^2}{2\nu\beta}
            + \dfrac{\beta}{4\alpha_o - 2(\nu+\gamma_v)},
            & i\in\mathcal{J}_v,\\[6pt]
            \dfrac{\bar\omega_1^2\dot{q}_{\max,i}}{2\nu\beta}
            + \dfrac{\beta/\dot{q}_{\max,i}}{4\alpha_o - 2(\nu+\gamma_p)},
            & i\in\mathcal{J}_p.
          \end{cases}
        \end{equation*}
        Let $\ddot{q}_{\mathrm{rob},A,i} \coloneqq -\ddot{q}_{s, i} -\hat{d}_{o,i} + \ddot{q}_{\mathrm{marg},i}$ (lower-bound direction) and $\ddot{q}_{\mathrm{rob},B,i} \coloneqq -\ddot{q}_{s, i} - \hat{d}_{o,i} - \ddot{q}_{\mathrm{marg},i}$ (upper-bound direction). Define the base torque $\boldsymbol{\tau}_{\mathrm{base}} = \boldsymbol{\tau}_h - \hat{\mathbf{C}}(\mathbf{q},\dot{\mathbf{q}})\dot{\mathbf{q}} - \hat{\mathbf{G}}(\mathbf{q}) - \hat{\mathbf{M}}(\mathbf{q})(\ddot{\mathbf{q}}_s - \hat{\mathbf{d}}_c - \hat{\mathbf{f}}(\boldsymbol{\xi}; \boldsymbol{\theta}))$, which collects the known dynamic terms and robust compensators. The nominal acceleration couples the impedance model to the plant state via
        \begin{equation}\label{eq: qnom}
            \ddot{\mathbf{q}}_{\mathrm{nom}} = \ddot{\mathbf{q}}_{\mathrm{imp}} + \mathbf{K}_v (\dot{\mathbf{q}}_{\mathrm{imp}} - \dot{\mathbf{q}}) + \mathbf{K}_p (\mathbf{q}_{\mathrm{imp}} - \mathbf{q}).
        \end{equation}
        Define $\ddot{\mathbf{q}}_{\mathrm{QP}} \coloneqq \ddot{\mathbf{q}}_{\mathrm{nom}} - \ddot{\mathbf{q}}_{\mathrm{cmd}}$. The commanded acceleration solves
        \begin{subequations}\label{eq: Torque Constraints}
            \begin{equation}
                (\ddot{\mathbf{q}}_{\mathrm{cmd}}, \mathbf{s}^*) = \argmin_{\ddot{\mathbf{q}},\, \mathbf{s}} \; \tfrac{1}{2}\| \ddot{\mathbf{q}} - \ddot{\mathbf{q}}_{\mathrm{nom}} \|^2 + \rho\, \mathbf{1}^\top \mathbf{s} + \tfrac{\varepsilon_s}{2} \|\mathbf{s}\|^2,
            \end{equation}
            \begin{align}
                \text{s.t.}\quad
                    & \boldsymbol{\alpha}_{A} + \ddot{\mathbf{q}}_{\mathrm{rob},A} \leq \ddot{\mathbf{q}} \leq \boldsymbol{\alpha}_{B} + \ddot{\mathbf{q}}_{\mathrm{rob},B}, \label{eq: robust_CBF} \\
                    & \boldsymbol{\tau}_{\min}  - \mathbf{s} \leq \hat{\mathbf{M}} \ddot{\mathbf{q}} - \boldsymbol{\tau}_{\mathrm{base}} \leq \boldsymbol{\tau}_{\max}  + \mathbf{s}, \label{eq: torque_bounds} \\
                    & \mathbf{s} \geq \mathbf{0}, \label{eq: slack_nonneg}
            \end{align}
        \end{subequations}
        where $\mathbf{s} \in \R^n$ relaxes the torque constraints, $\rho > 0$ is the exact penalty weight, and $\varepsilon_s > 0$ regularises the slack. The CBF constraints~\eqref{eq: robust_CBF} remain hard; only~\eqref{eq: torque_bounds} is softened.

        \par For sufficiently large $\rho$~\cite{kerrigan_soft_2000}, $\mathbf{s}^* = \mathbf{0}$ whenever the hard-constrained problem is feasible. Otherwise, $\mathbf{s}$ absorbs the minimum torque violation while~\eqref{eq: robust_CBF} stays active.

        \par The admissible wrench set $\mathcal{F}(\mathbf{x}) \subset \R^6$ from Section~\ref{sec:problem} is
        \begin{equation}
            \begin{aligned}
                \mathcal{F}(\mathbf{x}) \coloneqq \big\{ \mathbf{F}_h \in \R^6 : \exists\, \ddot{\mathbf{q}} \text{ s.t.} \\
                \eqref{eq: robust_CBF},\ \eqref{eq: torque_bounds},\ \mathbf{s}=\mathbf{0} \big\}.
            \end{aligned}
        \end{equation}
        The control torque is
        \begin{equation}\label{eq: control torque}
            \mathbf{u} = 
            \hat{\mathbf{M}}(\mathbf{q}) \ddot{\mathbf{q}}_{\mathrm{cmd}} - \boldsymbol{\tau}_{\mathrm{base}},
        \end{equation}
        which reduces to the nominal control~\eqref{eq: Nominal Control} when $\ddot{\mathbf{q}}_{\mathrm{cmd}} = \ddot{\mathbf{q}}_{\mathrm{nom}}$ (i.e., no constraints active).

    \subsection{Feasibility Metric}\label{subsec: feasibility}
        \par To monitor proximity to torque infeasibility online, define for each unit direction $\mathbf{d} \in \mathcal{D} \coloneqq \{\pm \mathbf{e}_1, \pm \mathbf{e}_2, \pm \mathbf{e}_3\}$ the largest wrench magnitude $F_{\mathbf{d}}(\mathbf{x})$ compatible with~\eqref{eq: robust_CBF} and zero-slack torque feasibility. Let $\boldsymbol{\tau}_{\mathrm{base},0}$ denote $\boldsymbol{\tau}_{\mathrm{base}}$ with the interaction wrench removed. For $\mathbf{F}_h = F\mathbf{d}$,
        \begin{equation}\label{eq: wrench_lp_torque}
            \boldsymbol{\tau}_{\min} \leq \hat{\mathbf{M}}(\mathbf{q})\ddot{\mathbf{q}} - \boldsymbol{\tau}_{\mathrm{base},0} + F\,\mathbf{J}(\mathbf{q})^\top \mathbf{d} \leq \boldsymbol{\tau}_{\max}.
        \end{equation}
        For fixed $\mathbf{x}$ and direction $\mathbf{d}$, the largest magnitude $F$ for which there exists an acceleration $\ddot{\mathbf{q}}$ satisfying both~\eqref{eq: robust_CBF} and~\eqref{eq: wrench_lp_torque} is obtained by the linear programme
        \begin{subequations}\label{eq: wrench_lp}
            \begin{equation}
                F_{\mathbf{d}}(\mathbf{x}) = \max_{F,\,\ddot{\mathbf{q}}}\; F \label{eq: wrench_lp_obj}
            \end{equation}
            \begin{align}
                \text{s.t.}\quad
                    & \boldsymbol{\alpha}_{A} + \ddot{\mathbf{q}}_{\mathrm{rob},A} \leq \ddot{\mathbf{q}} \leq \boldsymbol{\alpha}_{B} + \ddot{\mathbf{q}}_{\mathrm{rob},B}, \label{eq: wrench_lp_cbf}\\
                    & \hat{\mathbf{M}}(\mathbf{q})\ddot{\mathbf{q}} - \mathbf{J}(\mathbf{q})^\top \mathbf{d}\,F \leq \boldsymbol{\tau}_{\max} + \boldsymbol{\tau}_{\mathrm{base},0}, \label{eq: wrench_lp_ub}\\
                    & -\hat{\mathbf{M}}(\mathbf{q})\ddot{\mathbf{q}} + \mathbf{J}(\mathbf{q})^\top \mathbf{d}\,F \leq -\boldsymbol{\tau}_{\min} - \boldsymbol{\tau}_{\mathrm{base},0}, \label{eq: wrench_lp_lb}\\
                    & F \geq 0. \label{eq: wrench_lp_nonneg}
            \end{align}
        \end{subequations}
        Problem~\eqref{eq: wrench_lp} is evaluated for each $\mathbf{d} \in \mathcal{D}$, and the online metric is
        \begin{equation}\label{eq: wrench_metric}
            F_{\max}(\mathbf{x}) \coloneqq \min_{\mathbf{d}\in\mathcal{D}} F_{\mathbf{d}}(\mathbf{x}).
        \end{equation}
        When $\|\mathbf{F}_h(t)\| > F_{\max}(\mathbf{x}(t))$, zero-slack torque feasibility is lost and the soft-constrained QP activates $\mathbf{s}^* > 0$.

%% file: section/Theoretical-Results.tex
\section{Main Results}\label{sec:theory}
    \begin{lemma}\label{lem: Joint-Task Error}
        Consider a bounded, Lipschitz continuous reference $t \mapsto (\mathbf{q}_r(t), \dot{\mathbf{q}}_r(t), \ddot{\mathbf{q}}_r(t))$. Let $\mathbf{z} \coloneqq \mathbf{h}_{\mathbf{z}}(\mathbf{q}) \in \mathbb{R}^6$ and $\mathbf{J}(q)$ denote the Jacobian. Let $(\mathbf{q}, \dot{\mathbf{q}}) \in \mathcal{Q}$, where $\mathcal{Q}$ is a compact set and the functions $\mathbf{h}_{\mathbf{z}}, \mathbf{J}$ are smooth maps. Then $\exists\ c_1, c_2, c_3 \in \mathbb{R}_{>0}$: $\norm{\mathbf{h}_{\mathbf{z}}(\mathbf{q}_r) - \mathbf{h}_{\mathbf{z}}(\mathbf{q})} \leq c_1 \norm{\mathbf{q}_r - \mathbf{q}}$ and $\norm{\mathbf{J}(\mathbf{q}_r)\dot{\mathbf{q}}_r - \mathbf{J}(\mathbf{q})\dot{\mathbf{q}}} \leq c_2 \norm{\mathbf{q}_r - \mathbf{q}} +  c_3 \norm{\dot{\mathbf{q}}_r - \dot{\mathbf{q}}}$.
    \end{lemma}

    \begin{proof}
        Follows from Lipschitz Continuity and the Mean Value Theorem. Refer to \cite{khalil_nonlinear_2002}.
    \end{proof}

    \begin{lemma}[\cite{polycarpou_robust_1996}]\label{lem: Robust Bound}
        Let $a \in \mathbb{R},\ b \in \mathbb{R}_{> 0}$, and $\kappa \approx 0.2785$ satisfy $\kappa + \ln(\kappa) = -1$. Then for any $\varepsilon \in \mathbb{R}_{> 0}$: $0 \leq |a|b - ab \tanh{\left(\kappa \frac{ab}{\varepsilon}\right)} \leq \varepsilon$.
    \end{lemma}

    \begin{assumption}[\cite{wang_disturbance_2023}]\label{assum: valid_cbf}
        There exist positive constants $\gamma_p,\,\gamma_v,\,\alpha_o,\,\beta$, such that
        \begin{equation*}
            \begin{aligned}
                & \sup_{\ddot{\mathbf{q}}_{\mathrm{cmd}}}\Big[\mathcal{L}_f^W({^i}h_{(\cdot)}) - \|\mathcal{L}_g^W({^i}h_{(\cdot)}) \| \mathcal{W} - \tfrac{\bar\omega_1^2}{2\nu\beta} \\
                & \quad + \tfrac{\beta \|\mathcal{L}_g^W({^i}h_{(\cdot)}) \|^2}{4\alpha_o - 2(\nu+\gamma_{(\cdot)})} + \gamma_{(\cdot)} {^i}h_{(\cdot)} \\
                & \quad + \mathcal{L}_g^W({^i}h_{(\cdot)}) (\ddot{\mathbf{q}}_{\mathrm{cmd}} + \ddot{\mathbf{q}}_{\mathrm{s}}) \Big] \geq 0,
            \end{aligned}
        \end{equation*}
        where $\mathcal{L}_f^W$ and $\mathcal{L}_g^W$ are the weak set-valued Lie derivatives of the closed-loop system, and
        \[
            \mathcal{W} = \bar{\omega}_0 + \sqrt{\bar{\omega}_0^2 + \tfrac{\bar{\omega}_1}{\nu(2\alpha_o - \nu)}}.
        \]
    \end{assumption}

    \begin{theorem}\label{Thm: Main Result}
        Consider~\eqref{eq: Robot Model}--\eqref{eq: control torque} under Assumptions~\ref{assum: Scenario}--\ref{assum: valid_cbf}. For all $t \in \R_{>0}$, let the update laws  with $\Gamma$-Projection~\cite{lavretsky_projection_2012} and 
        $\sigma$-modification be
        \begin{align}
            \dot{\boldsymbol{\theta}}(t) &= \operatorname{Proj}_{\boldsymbol{\Gamma}_\theta}\!\big(\boldsymbol{\theta}(t), -\boldsymbol{\zeta}(\boldsymbol{\xi}) \mathbf{r}(t), f(\boldsymbol{\theta}) \big), \label{eq: theta law} \\
            \dot{\hat{\mathbf{w}}}(t) &= \boldsymbol{\Gamma}_{w}|\mathbf{r}(t)| - \sigma_{w} \boldsymbol{\Gamma}_{w}\hat{\mathbf{w}}(t), \label{eq: w Law}
        \end{align}
        where $\boldsymbol{\Gamma}_{\theta} = \operatorname{diag}(\gamma_{\theta 1}, \dots, \gamma_{\theta n}) \otimes \mathbf{I}_{m\cdot M} \in \mathbb{R}^{N\times N}$ and $\boldsymbol{\Gamma}_{w} = \operatorname{diag}(\gamma_{w 1}, \dots, \gamma_{w n}) \in \mathbb{R}^{n\times n}$ are positive-definite gains, $\sigma_{w} \in \mathbb{R}_{> 0}$ is a small leakage parameter, and $f(\boldsymbol{\theta}) \coloneqq \tfrac{\norm{\boldsymbol{\theta}}^2 - \Theta^2}{2 \varepsilon \Theta + \varepsilon^2}$ is a convex bounding function with nominal bound $\Theta > 0$ (Assumption~\ref{assum: bounds}) and tolerance $\varepsilon > 0$. Let the DOB updates be as in~\eqref{eq: dob_update}, and define the SMC term be
        \begin{equation}\label{eq: Smooth SMC}
            \ddot{q}_{si} = \hat{w}_i \tanh{\!\left(\frac{\kappa}{\varepsilon_i} r_i \hat{w}_i \right)},
        \end{equation}
        with $\varepsilon_i \in \R_{>0}$. Suppose $\forall t \in \R_{> 0},\ (\mathbf{x}, \mathbf{F}_h) \in \mathcal{A} \coloneqq \{(\mathbf{x}, \mathbf{F}_h) \in \R^{2n + 6} : \mathbf{F}_h \in \mathcal{F}(\mathbf{x}) \}$, where the $6$ in $\R^{2n+6}$ is the dimension of the wrench $\mathbf{F}_h \in \R^6$. Then, the following statements hold:
        \begin{enumerate}
            \renewcommand{\labelenumi}{(\roman{enumi})}
            \item $\mathbf{u} \in \mathcal{U},\, \forall t \in \mathbb{R}_{\geq 0}$, \label{stat: 1}
            \item If $\mathbf{x}(0) \in \mathcal{C}^* \subseteq \mathcal{C}_0$, then $\mathbf{x}(t) \in \mathcal{C}^*,\, \forall t \in \mathbb{R}_{\geq 0}$, \label{stat: 2}
            \item The tracking error of the desired impedance trajectory $\mathbf{z}_{\mathrm{imp}}(t) - \mathbf{z}(t)$ is uniformly ultimately bounded. \label{stat: 3}
        \end{enumerate}
    \end{theorem}

    \begin{proof}
        See Appendix~\ref{appendix: proof}.
    \end{proof}

    \begin{remark}
        Dropping the leakage term \eqref{eq: w Law} and replacing \eqref{eq: Smooth SMC} with a discontinuous SMC yields asymptotic stability of the impedance tracking error (when $\mathbf{x}_{\mathrm{imp}}(t) \in \mathcal{X}_b$ and the torque constraints are inactive), but Filippov solutions and a redefinition of the NCBFs would be required. In addition, the discontinuous nature of the new control law induces high-frequency oscillations in the control signal, leading to chattering, which can damage the actuators.
    \end{remark}

    \par Forward invariance via Glotfelter's NCBF theorem~\cite{glotfelter_nonsmooth_2017} requires the control law to admit at least Carath\'eodory solutions, which presupposes local Lipschitz continuity. Because the QP~\eqref{eq: Torque Constraints} involves the state-dependent inertia $\hat{\mathbf{M}}(\mathbf{q})$ in its constraint matrix, the standard piecewise-affine conclusions for multiparametric QPs do not apply directly~\cite{agrawal_reformulations_2025}. The following proposition confirms that the QP solution retains the required regularity.

    \begin{proposition}\label{prop: Lipschitz QP}
        Let $\mathcal{X} \subset \R^{2n}$ be compact and chosen to exclude kinematic singularities, and let $\mathbf{z} = (\ddot{\mathbf{q}}, \mathbf{s}) \in \R^{2n}$ be the decision vector of the QP~\eqref{eq: Torque Constraints}. Write the QP in parametric form as
        \[
            \mathbf{z}^*(\mathbf{x}) = \argmin_{\mathbf{z}} \; \tfrac{1}{2}\mathbf{z}^\top \mathbf{H} \mathbf{z} + \mathbf{h}(\mathbf{x})^\top \mathbf{z} \quad \text{s.t.}\quad \mathbf{A}(\mathbf{x})\mathbf{z} \leq \mathbf{b}(\mathbf{x}),
        \]
        where $\mathbf{H} = \diag(\mathbf{I}_n,\, \varepsilon_s \mathbf{I}_n) \succ \mathbf{0}$ is the constant Hessian, and $\mathbf{h}(\mathbf{x}), \mathbf{A}(\mathbf{x}), \mathbf{b}(\mathbf{x})$ collect the state-dependent QP data. Suppose the following hold on $\mathcal{X}$:
        \begin{enumerate}
            \renewcommand{\labelenumi}{(P\arabic{enumi})}
            \item (Uniform strong convexity) $\mathbf{H} \succeq \mu \mathbf{I}$ for some $\mu > 0$. \label{prop: P1}
            \item (Lipschitz data) $\mathbf{A}(\cdot), \mathbf{b}(\cdot), \mathbf{h}(\cdot)$ are Lipschitz on~$\mathcal{X}$. \label{prop: P2}
            \item (Uniform feasibility) The feasible set $\{\mathbf{z} : \mathbf{A}(\mathbf{x})\mathbf{z} \leq \mathbf{b}(\mathbf{x})\} \neq \emptyset$ for all $\mathbf{x} \in \mathcal{X}$, with $\boldsymbol{\alpha}_{A,i} + \ddot{q}_{\mathrm{rob},A,i} < \boldsymbol{\alpha}_{B,i} + \ddot{q}_{\mathrm{rob},B,i}$ element-wise. \label{prop: P3}
            \item (MFCQ) The Mangasarian--Fromovitz constraint qualification \cite{bednarczuk_mangasarian-fromovitz-type_2024} holds at $\mathbf{z}^*(\mathbf{x})$ for all $\mathbf{x} \in \mathcal{X}$. \label{prop: P4}
        \end{enumerate}
        Then $\mathbf{z}^*(\mathbf{x})$, and thus the commanded acceleration $\ddot{\mathbf{q}}_{\mathrm{cmd}}(\mathbf{x})$, are locally Lipschitz on $\mathcal{X}$.
    \end{proposition}

    \begin{proof}
        Uniform strong convexity \ref{prop: P1} holds by construction since $\varepsilon_s > 0$. Local Lipschitz continuity of $\mathbf{z}^*(\mathbf{x})$ then follows from classical sensitivity theory for strongly convex parametric optimisation problems (see~\cite{bonnans_perturbation_2000}). Since $\ddot{\mathbf{q}}_{\mathrm{cmd}} = [\mathbf{I}_n \;\; \mathbf{0}]\mathbf{z}^*$, the result extends to $\ddot{\mathbf{q}}_{\mathrm{cmd}}$ by linearity.
    \end{proof}

    \subsection{Verification of Lipschitz Continuity}
    \par We verify that the control law \eqref{eq: control torque} constructed in Section~\ref{sec:design} satisfies Assumptions~\ref{prop: P2}--\ref{prop: P4} of Proposition~\ref{prop: Lipschitz QP} on the compact operating set $\mathcal{X}$.
    
    \subsubsection{Lipschitz Data (Assumption~\ref{prop: P2})}
        For the constraint data $\mathbf{A}(\cdot)$ and $\mathbf{b}(\cdot)$ to be Lipschitz on $\mathcal{X}$, each constituent function forming the bounds $\ddot{\mathbf{q}}_{\mathrm{lower}}(\mathbf{x})$ and $\ddot{\mathbf{q}}_{\mathrm{upper}}(\mathbf{x})$ must itself be locally Lipschitz. We verify this for each component:

        \emph{IT2-FLS ($\hat{\mathbf{f}}$):} The approximator uses Gaussian membership functions and the Nie-Tan type-reduction operator, both $C^\infty$ with respect to their inputs. The FLS inputs $\boldsymbol{\xi}_i = (\sin(q_i), \cos(q_i), \dot{q}_i, \bar{\xi}_i)$ are smooth bounded mappings of the state, so $\hat{\mathbf{f}}(\mathbf{x})$ is locally Lipschitz on $\mathcal{X}$.

        \emph{SMC ($\ddot{\mathbf{q}}_s$):} The smooth compensator $\ddot{q}_{si} = \hat{w}_i \tanh(\kappa r_i \hat{w}_i / \varepsilon_i)$ is $C^\infty$ and globally Lipschitz, by design avoiding the standard discontinuous signum function.

        \emph{DOB ($\hat{\mathbf{d}}_c, \hat{\mathbf{d}}_o$):} The observers~\eqref{eq: dob} are affine in $\dot{\mathbf{q}}$ with smooth internal states, hence locally Lipschitz.

        \emph{NCBF bounds ($\boldsymbol{\alpha}_A, \boldsymbol{\alpha}_B$):} These are composed using $R_1(\cdot) = \max\{0, \cdot\}$ and $R_2(\cdot) = \max\{0, -(\cdot)\}$, which are globally Lipschitz with constant~$1$. Therefore, $\boldsymbol{\alpha}_A$ and $\boldsymbol{\alpha}_B$ are locally Lipschitz.

        \emph{Inertia ($\hat{\mathbf{M}}$):} Since $\mathcal{X}$ excludes kinematic singularities, $\hat{\mathbf{M}}(\mathbf{q})$ is $C^1$ and strictly positive-definite on an open neighbourhood of $\mathcal{X}$, hence Lipschitz on $\mathcal{X}$.
    \hfill \QED

    \subsubsection{Uniform Feasibility (Assumption~\ref{prop: P3})}
        Because the slack variables $\mathbf{s} \in \R^n_{\geq 0}$ can absorb arbitrary torque-constraint violations, uniform feasibility of the QP reduces to showing that the hard CBF acceleration box remains strictly non-degenerate: $\ddot{q}_{\mathrm{lower}, i}(\mathbf{x}) < \ddot{q}_{\mathrm{upper}, i}(\mathbf{x})$ for all $\mathbf{x} \in \mathcal{X}$, i.e.\ the width $\Delta \ddot{q}_i \coloneqq (\alpha_{B, i} - \alpha_{A, i}) - 2\ddot{q}_{\mathrm{marg}, i} > 0$. This condition concerns state-side feasibility only; torque infeasibility under large interaction wrenches is handled separately by the slack variables and is the subject of the infeasibility simulation in Section~\ref{sec:sim}.

        \emph{Velocity constraints ($i \in \mathcal{J}_v$):} The CBF bounds yield the constant width $\alpha_{vB,i} - \alpha_{vA,i} = 2\gamma_v \dot{q}_{\max,i}$, so the non-degeneracy condition is
        \begin{equation}\label{eq: vel_feasibility}
            \gamma_v \dot{q}_{\max,i} > \ddot{q}_{\mathrm{marg},v,i}(\gamma_v).
        \end{equation}

        \emph{Position constraints ($i \in \mathcal{J}_p$):} The width varies with the state. Exploiting the geometry of the composed NCBFs---specifically that $h_{p1} + h_{p2} \geq 1$ on $\mathcal{C}^*$---and evaluating at the critical boundary yields the analytical minimum:
        \[
            \alpha_{pB,i} - \alpha_{pA,i} \geq \dot{q}_{\max,i}\!\left[\gamma_p - l\frac{\dot{q}_{\max,i}}{q_{\max,i}} K(l)\right],
        \]
        where $K(l) = \left(\tfrac{l-1}{2l-1}\right)^{\frac{l-1}{l}} \left(\tfrac{l}{2l-1}\right)$. The non-degeneracy condition becomes
        \begin{equation}\label{eq: pos_feasibility}
            \gamma_p \dot{q}_{\max,i} > 2\ddot{q}_{\mathrm{marg},p,i}(\gamma_p) + l\frac{\dot{q}_{\max,i}^2}{q_{\max,i}} K(l).
        \end{equation}

        \emph{Parameter feasibility:} The robust margin $\ddot{q}_{\mathrm{marg},i}(\gamma)$ is not static; it depends on $\gamma$ via $\ddot{q}_{\mathrm{marg}}(\gamma) = C_1 + C_2 / (2\alpha_o - \nu - \gamma)$, where $C_1 = \bar{\omega}_1^2/(2\nu\beta)$ and $C_2 = \beta/2$. Because $\ddot{q}_{\mathrm{marg}} \to \infty$ as $\gamma \to (2\alpha_o - \nu)$, conditions~\eqref{eq: vel_feasibility}--\eqref{eq: pos_feasibility} define coupled inequalities.

        \emph{Velocity case:} Multiplying~\eqref{eq: vel_feasibility} by the positive factor $(2\alpha_o - \nu - \gamma_v)$ yields the quadratic:
        \begin{equation}\label{eq: vel_quad}
            \begin{aligned}
                \dot{q}_{\max,i}\gamma_v^2 &- \!\left(\dot{q}_{\max,i}(2\alpha_o - \nu) + C_1\right)\gamma_v + \\ 
                & \!\left(C_1(2\alpha_o - \nu) + C_2\right) < 0.
            \end{aligned}
        \end{equation}
        A feasible tuning interval $\gamma_v \in (\gamma_{v,\min}, \gamma_{v,\max})$ exists if and only if the discriminant of~\eqref{eq: vel_quad} is strictly positive, which holds when the physical joint kinematic limits $\dot{q}_{\max, i}$ sufficiently outweigh the worst-case disturbance bounds $C_1, C_2$---a condition met by appropriately scaled robotic hardware and observer gains.

        \emph{Position case:} Define $C_3 \coloneqq l\,\dot{q}_{\max, i}^2\, K(l) / q_{\max, i}$ and $\bar{C}_1 \coloneqq 2C_1 \dot{q}_{\max, i} + C_3$, where the factor $2C_1 \dot{q}_{\max, i}$ accounts for the $\dot{q}_{\max, i}$-scaling of the robust margin in the position case, and $C_3$ captures the state-dependent NCBF curvature term from~\eqref{eq: position NCBF}. Multiplying~\eqref{eq: pos_feasibility} by $(2\alpha_o - \nu - \gamma_p) > 0$ and rearranging gives the quadratic:
        \begin{equation}\label{eq: pos_quad}
            \begin{aligned}
                \dot{q}_{\max,i}\gamma_p^2 &- \!\left(\dot{q}_{\max,i}(2\alpha_o - \nu) + \bar{C}_1\right)\gamma_p + \\ 
                & \!\left(\bar{C}_1(2\alpha_o - \nu) + \tfrac{2C_2}{\dot{q}_{\max,i}}\right) < 0.
            \end{aligned}
        \end{equation}
        By the same discriminant argument, a feasible interval $\gamma_p \in (\gamma_{p,\min}, \gamma_{p,\max})$ exists when the joint limits outweigh the disturbance bounds. The additional curvature term $C_3$ narrows the feasible interval relative to the velocity case, requiring either a smaller $l$ or larger $\alpha_o$ to compensate.
    \hfill \QED

    \subsubsection{MFCQ (Assumption~\ref{prop: P4})}
        MFCQ requires the existence of a strict descent direction $\mathbf{d} \in \R^{2n}$ for all active constraints at $\mathbf{z}^*(\mathbf{x})$.

        \emph{Acceleration bounds:} By \ref{prop: P3}, the box constraints are strictly non-degenerate, so the upper and lower bounds cannot be simultaneously active for the same joint. An inward-pointing direction $\mathbf{d}_{\ddot{q}}$ trivially exists for any active face of the hyper-rectangle.

        \emph{Torque constraints:} These constraints are decoupled from the strict feasibility set via the independent slack variables $\mathbf{s}$. For any active torque constraint, a direction $d_{s, i} > 0$ strictly decreases the constraint residual without violating $s_i \geq 0$.

        \par Fix any inward-pointing direction $\mathbf{d}_{\ddot{q}}$ for the active acceleration faces. For each torque constraint $i$, choose $d_{s, i} > \norm{\hat{\mathbf{M}}_{i,:} \mathbf{d}_{\ddot{q}}}$, so that the directional derivative of the corresponding constraint is strictly negative.

        \par Combining $\mathbf{d}_{\ddot{q}}$ with the positive slack directions yields a global $\mathbf{d}$ that strictly decreases all active constraints, so MFCQ holds at $\mathbf{z}^*(\mathbf{x})$ for all $\mathbf{x} \in \mathcal{X}$.

%% file: section/Tuning-Guidelines.tex
\section{Tuning Guidelines}\label{sec:tuning}
    \par Table~\ref{tab:tuning} lists the main design parameters. The paragraphs below state the conditions behind the safety and stability guarantees in Section~\ref{sec:theory}; choose $\mathcal{C}_0$ and the FLS structure so that $\mathcal{C}^* \subseteq \mathcal{C}_0$ on the intended operating envelope.

    \begin{table}[!t]
        \centering
        \caption{Summary of tunable design parameters.}
        \label{tab:tuning}
        \renewcommand{\arraystretch}{1.15}
        \footnotesize
        \begin{tabularx}{\columnwidth}{@{}c l X@{}}
            \toprule
            \textbf{Symbol} & \textbf{Module} & \textbf{Role} \\
            \midrule
            $\mathbf{M}_d, \mathbf{B}_d, \mathbf{K}_d$ & Impedance & Desired task-space inertia, damping, and stiffness \\
            $\mathbf{K}_p, \mathbf{K}_v$ & Ref.\ model & Coupling between plant and command-driven reference model \\
            $l$ & NCBF & Sharpness of the position--velocity safe-set approximation \\
            $\gamma_v, \gamma_p$ & NCBF & Convergence rates of the velocity and position barriers \\
            $\mathbf{K}_r, \boldsymbol{\Lambda}$ & Adaptation & Gains of the filtered reference-model tracking error $\mathbf{r}$ \\
            $\boldsymbol{\Gamma}_{\theta}$ & FLS & Adaptation rate of the consequent weights \\
            $\boldsymbol{\Gamma}_{w}$ & SMC & Adaptation rate of the residual bound estimate $\hat{\mathbf{w}}$ \\
            $\sigma_w$ & SMC & Leakage in the $\hat{\mathbf{w}}$ update law \\
            $\varepsilon_i$ & SMC & Boundary-layer thickness of the smooth compensator \\
            $\alpha_c$ & DOB & Gain of the compensation disturbance observer \\
            $\alpha_o$ & DOB & Gain of the safety disturbance observer \\
            $\nu$ & Robust NCBF & Coupling gain in the observer-based CBF margin \\
            $\beta$ & Robust NCBF & Scaling of the robust acceleration margin \\
            $\bar{\omega}_1$ & Robust NCBF & Conservative bound on $\|\dot{\mathbf{d}}_{\mathrm{CBF}}\|$ \\
            $\rho$ & QP & Exact-penalty weight on torque slack \\
            $\varepsilon_s$ & QP & Regularisation on the slack variables \\
            \bottomrule
        \end{tabularx}
    \end{table}

    \par \textbf{Reference model and adaptation.} The command-driven reference model requires positive-definite diagonal gains satisfying $\mathbf{K}_v = \boldsymbol{\Lambda} + \mathbf{K}_r$ and $\mathbf{K}_p = \mathbf{K}_r \boldsymbol{\Lambda}$. The stability proof of Theorem~\ref{Thm: Main Result}(iii) treats $(\mathbf{K}_r, \boldsymbol{\Lambda})$ as the natural-frequency and damping of the filtered error dynamics; larger $\mathbf{K}_r$ yields faster reference-model tracking but also amplifies the effect of any unmodelled QP correction through $\ddot{\mathbf{q}}_{\mathrm{QP}}$. The FLS and SMC gains $(\boldsymbol{\Gamma}_\theta, \boldsymbol{\Gamma}_w, \sigma_w, \varepsilon_i)$ enter the same Lyapunov analysis: $\boldsymbol{\Gamma}_\theta$ and $\boldsymbol{\Gamma}_w$ must be positive definite, $\sigma_w > 0$ should be kept small to limit parameter drift, and each $\varepsilon_i > 0$ sets the residual approximation error that bounds the SMC layer in Lemma~\ref{lem: Robust Bound}. The compensation observer gain $\alpha_c > 0$ should be chosen large enough for effective disturbance rejection, but need only be consistent with the boundedness assumptions on $\mathbf{d}(t)$ in Assumption~\ref{assum: bounds}.

    \par \textbf{Safety and robust NCBF parameters.} State invariance (Theorem~\ref{Thm: Main Result}(ii)) rests on Assumption~\ref{assum: valid_cbf} and the hard CBF constraints~\eqref{eq: robust_CBF}. The robust margin depends on $(\nu, \beta, \bar{\omega}_1, \alpha_o, \gamma_v, \gamma_p)$ through~\eqref{eq: Torque Constraints} and must satisfy $0 < \nu < 2\alpha_o$ together with $\max\{\gamma_v, \gamma_p\} < 2\alpha_o - \nu$. Increasing $\bar{\omega}_1$ or decreasing $\beta$ enlarges the margin $\ddot{q}_{\mathrm{marg},i}$ and makes the safety filter more conservative; increasing $\alpha_o$ allows larger $\gamma_v$ and $\gamma_p$ before the margin diverges. The NCBF rates $\gamma_v$ and $\gamma_p$ should be selected inside the feasible intervals implied by~\eqref{eq: vel_quad} and~\eqref{eq: pos_quad}. A positive discriminant in each quadratic defines an admissible interval $(\gamma_{(\cdot),\min}, \gamma_{(\cdot),\max})$; larger $\gamma$ improves constraint convergence but consumes more of the available acceleration box through $\ddot{q}_{\mathrm{marg},i}$. For position-constrained joints, the exponent $l$ in~\eqref{eq: position NCBF} trades approximation of the true safe set against curvature of the barrier: larger even $l$ makes $\mathcal{C}^*(l)$ approach $\mathcal{C}$, but also increases the term $C_3$ in~\eqref{eq: pos_quad} and narrows the feasible range of $\gamma_p$.

    \par \textbf{QP feasibility and torque constraints.} Uniform feasibility of the hard CBF box (Proposition~\ref{prop: Lipschitz QP}(P3)) requires, for every joint $i$, that the robust lower and upper acceleration bounds remain strictly separated: $\ddot{q}_{\mathrm{lower},i}(\mathbf{x}) < \ddot{q}_{\mathrm{upper},i}(\mathbf{x})$ for all $\mathbf{x}$ in the operating set. This condition is most restrictive near the constraint boundary and for joints in $\mathcal{J}_p$, where both position and velocity limits interact. When a feasible interval for $\gamma_v$ or $\gamma_p$ exists, any interior choice preserves a nonempty hard constraint set for the QP. Torque satisfaction whenever hard feasibility holds (Theorem~\ref{Thm: Main Result}(i)) requires the exact penalty weight $\rho$ to exceed the optimal Lagrange multiplier of the hard-constrained problem; in practice, $\rho$ should be increased until torque slack remains inactive during nominal operation. The regularisation $\varepsilon_s > 0$ ensures strong convexity of the QP and the local Lipschitz continuity of $\ddot{\mathbf{q}}_{\mathrm{cmd}}$; it should be kept small so that slack activation remains a last resort rather than a persistent softening of the torque bounds.

    \par \textbf{Operating regime.} Theorem~\ref{Thm: Main Result} applies on $(\mathbf{x}, \mathbf{F}_h) \in \mathcal{A}$. If the interaction wrench exceeds what the torque limits can support, slack activation is unavoidable; begin tuning from the joint limits and the feasible intervals for $(\nu, \beta, \bar{\omega}_1, \alpha_o, \gamma_v, \gamma_p, l)$.

%% file: section/Simulation.tex
\begin{figure*}[t]
  \centering
  \includegraphics[width=\textwidth]{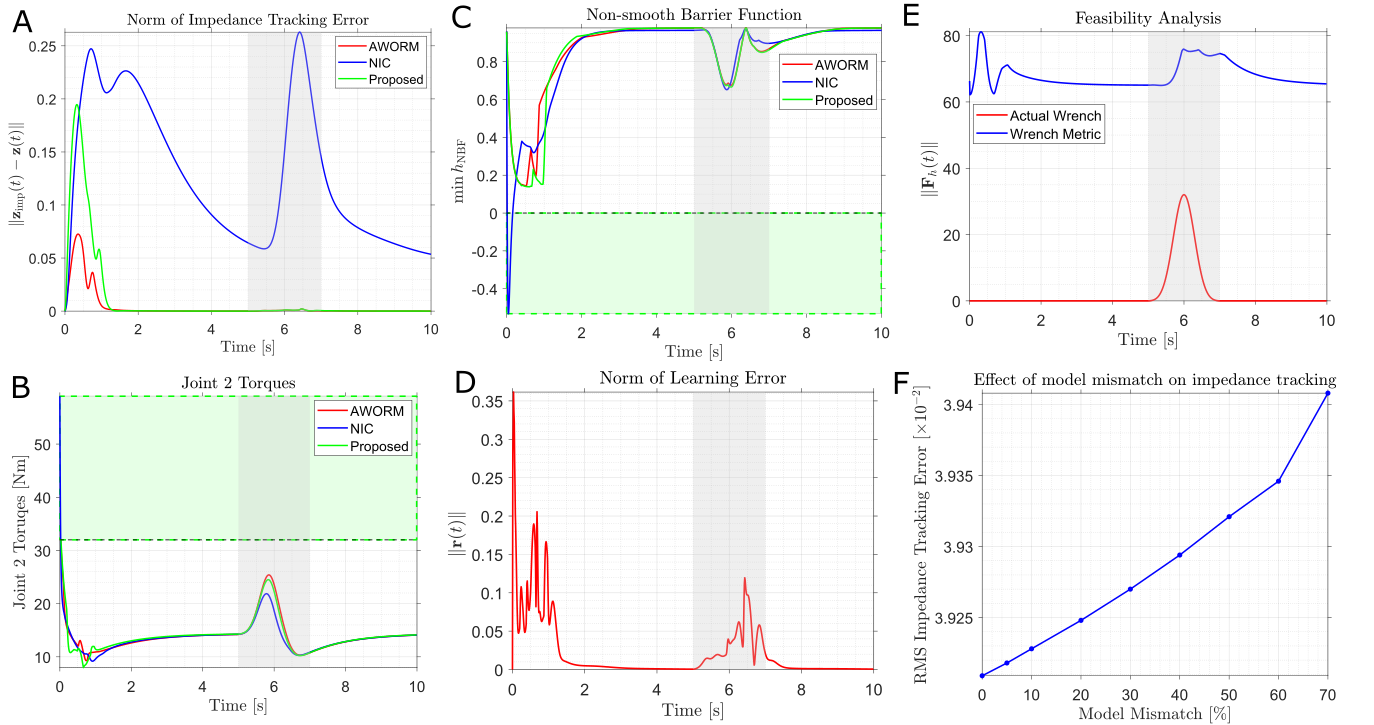}
  \caption{\textbf{A}: Impedance tracking error $\norm{\mathbf{z}_{\mathrm{imp}}(t)-\mathbf{z}(t)}$ for the proposed controller (green), AWORM (red), and NIC (blue). The shaded region denotes the interval of interaction wrenches. \textbf{B}: Joint~2 control torques under nominal limits. Dashed lines show $\tau_{\max}$; the green region indicates torque constraint violations. \textbf{C}: Minimum NCBF over all joints. Dashed horizontal lines show the safety threshold; the green region indicates safety violations. \textbf{D}: Norm of the reference-model learning error $\norm{\mathbf{r}(t)}$ for the proposed controller. \textbf{E}: Applied wrench norm $\norm{\mathbf{F}_h(t)}$ against the wrench metric $F_{\max}(\mathbf{x}(t))$ from~\eqref{eq: wrench_metric}. \textbf{F}: RMS impedance tracking error over the full simulation horizon as model mismatch increases from $0\%$ to $70\%$.}
  \label{fig:images}
\end{figure*}

\section{Simulation Validation}\label{sec:sim}
    \subsection{Setup}
    \par We validate the proposed controller on a simulated KINOVA Gen3 7-DOF manipulator. The true plant follows~\eqref{eq: Robot Model} with manufacturer inertial parameters. In simulation, the disturbance is set to joint friction, $\boldsymbol{\tau}_d(t)=\boldsymbol{\tau}_{\mathrm{fric}}(\dot{\mathbf{q}})$, modelled by the smooth element-wise approximation
    \begin{equation}\label{eq: friction}
        \begin{aligned}
            \tau_{\mathrm{fric},i}(\dot{q}_i) &= \bigl(f_{c,i} + (f_{s,i}-f_{c,i}) e^{-(\dot{q}_i/v_{s,i})^2}\bigr)\tanh(\kappa \dot{q}_i) \\
            &\quad + f_{v,i}\dot{q}_i, \qquad i \in [n],
        \end{aligned}
    \end{equation}
    with $\kappa = 25$\,rad$^{-1}$s. The coefficients are arbitrary but physically reasonable joint-wise values, not identified parameters.

    \par The controller uses only~\eqref{eq: Estimated Model} with $(\hat{\mathbf{M}}, \hat{\mathbf{C}}, \hat{\mathbf{G}})$ obtained by uniformly perturbing each physical parameter by up to $\pm 50\%$; neither~\eqref{eq: friction} nor $\boldsymbol{\tau}_{\mathrm{fric}}$ appears in the nominal model.

    \par The task is to regulate a fixed end-effector pose $\mathbf{z}_d$ through the desired impedance model $\mathbf{M}_d = \mathbf{I}_6$, $\mathbf{K}_d = \diag(49\mathbf{I}_3, 16\mathbf{I}_3)$, and $\mathbf{B}_d = \diag(49\mathbf{I}_3, 12\mathbf{I}_3)$ while rejecting a measured human wrench. The interaction force $\mathbf{F}_h(t)$ has peak components $[20, 20, -15]$\,N along the end-effector translational axes, zero moment components, and is shaped by a half-sine pulse with a Gaussian envelope over $[5, 7]$\,s. All simulations use a fixed-step RK4 integrator with $\Delta t = 10$\,ms over $10$\,s, starting from $\mathbf{q}(0) = [0; -1.0; 1.0; -0.5; 0.5; -1.2; 0]$ and $\dot{\mathbf{q}}(0) = \mathbf{0}$, with $\mathbf{x}(0) \in \mathcal{C}^* \subseteq \mathcal{C}_0$. Impedance tracking performance is reported as $\norm{\mathbf{z}_{\mathrm{imp}}(t)-\mathbf{z}(t)}$, with orientation evaluated via the quaternion error vector used in simulation.

    \par \textbf{Constraints.} Joints $2$, $4$, and $6$ are subject to position and velocity limits ($\mathcal{J}_p = \{2,4,6\}$, $q_{\max,i} = [126^\circ, 147^\circ, 117^\circ]$, $\dot{q}_{\max,i} = [17, 17, 25]$\,rpm); all other joints have velocity limits only ($\mathcal{J}_v = \{1,3,5,7\}$, $\dot{q}_{\max,i} = [17, 17, 25, 25]$\,rpm). Unless stated otherwise, torque limits are $\pm 32$\,Nm on joints $1$--$4$ and $\pm 13$\,Nm on joints $5$--$7$.

    \par \textbf{Controller parameters.} The IT2-FLS uses four inputs per joint, $(\sin q_i, \cos q_i, \dot{q}_i, \bar{\xi}_i)$, three Gaussian membership functions per input ($M=81$ rules), and first-order Sugeno consequents with five parameters per rule, giving $2835$ adapted parameters in total. The reference-model and adaptation gains are $\mathbf{K}_r = \boldsymbol{\Lambda} = \sqrt{250}\,\mathbf{I}_7$. The NCBF uses $l=6$, $\gamma_p = \gamma_v = 10$, $\beta = 35$, and $\bar{\omega}_1 = 75$. The DOB gains are $\alpha_c = 70$ and $\alpha_o = 80$; the SMC leakage is $\sigma_w = 0.0125$. The soft-constrained QP uses $\rho = 1000$, $\varepsilon_s = 10^{-6}$, and is solved online in MATLAB with qpOASES and warm-starting at default tolerances; the average qpOASES solve time is $5.4 \times 10^{-5}$\,s per control step.

    \par \textbf{Baselines.} We compare against:
    \begin{enumerate}\renewcommand{\labelenumi}{(\roman{enumi})}
        \item \textbf{NIC}: a nominal impedance controller using only the perturbed model, with no FLS, DOB, or safety filter;
        \item \textbf{AWORM}: the same FLS, DOB, and safety filter as the proposed method, but with adaptation driven by the impedance filtered error $\mathbf{r}_{\mathrm{imp}} = \dot{\mathbf{e}}_{\mathrm{imp}} + \boldsymbol{\Lambda}\mathbf{e}_{\mathrm{imp}}$ for $\mathbf{e}_{\mathrm{imp}} = \mathbf{q}_{\mathrm{imp}} - \mathbf{q}$, rather than the filtered reference-model tracking error $\mathbf{r}$.
    \end{enumerate}

    \par \textbf{Scenarios.} Four simulation studies are reported:
    \begin{enumerate}\renewcommand{\labelenumi}{(\roman{enumi})}
        \item \emph{Nominal comparison} (Fig.~\ref{fig:images}A--C): all three controllers under the nominal torque limits above;
        \item \emph{Model-mismatch sweep} (Fig.~\ref{fig:images}F): proposed controller only, with perturbation level varied from $0\%$ to $70\%$;
        \item \emph{Bursting comparison} (Fig.~\ref{fig:images_2}A--B): proposed versus AWORM with tightened limits of $\pm 22.5$\,Nm on joints $1$--$4$;
        \item \emph{Infeasibility test} (Fig.~\ref{fig:images_2}C--F): proposed controller only, with limits further reduced to $\pm 15$\,Nm on joints $1$--$4$.
    \end{enumerate}

    \subsection{Results}

    \subsubsection{Nominal Comparison}
    \par Fig.~\ref{fig:images}A compares the impedance tracking error $\norm{\mathbf{z}_{\mathrm{imp}}(t)-\mathbf{z}(t)}$ for all three controllers under the nominal torque limits. The NIC performs significantly worse throughout the run. Without adaptation or a safety filter, it cannot compensate for the combined effect of the $50\%$ parametric mismatch and the unmodeled friction~\eqref{eq: friction}, so its computed torques are systematically biased and impedance regulation degrades before and during the interaction interval. Both adaptive controllers achieve much lower error because the IT2-FLS and DOB learn and reject the resulting structured and unstructured residuals online; the proposed method tracks most closely.

    \par Fig.~\ref{fig:images}B--C show that NIC violates torque and state constraints, whereas both adaptive controllers remain within bounds and maintain $\min_i {^i h(\mathbf{x}_i)} \geq 0$. NIC departs from $\mathcal{C}^*$ immediately because it has no safety filter and applies torques from the mismatched nominal model despite $\mathbf{x}(0) \in \mathcal{C}^*$.

    \subsubsection{Effect of Model Mismatch}
    \par Fig.~\ref{fig:images}F reports the RMS impedance tracking error over the full $10$\,s trajectory as parameter perturbation increases from $0\%$ to $70\%$. Unmodeled friction~\eqref{eq: friction} is present in every run, so the error grows with mismatch but remains on the order of $10^{-2}$ even at $70\%$.

    \subsubsection{Constraint-Aware Adaptation versus AWORM Bursting}
    \par Fig.~\ref{fig:images_2}A--B isolates the effect of the command-driven reference model by comparing the proposed controller against AWORM with torque limits of $\pm 22.5$\,Nm on joints $1$--$4$. Both controllers share the same FLS, DOB, and safety filter; the only difference is the signal used for adaptation.

    \par AWORM exhibits the bursting behaviour characteristic of constrained learning-based systems. When the NCBFs activate, the safety filter suppresses motion and the impedance error grows. Because AWORM updates the FLS on $\mathbf{r}_{\mathrm{imp}}$, it interprets this constraint-induced deficit as impedance-model tracking error and adapts against the safety filter. When the state re-enters the interior of $\mathcal{C}^*$, the stored adaptation energy is released abruptly, producing the high-frequency torque oscillations in Fig.~\ref{fig:images_2}B and the degraded tracking in Fig.~\ref{fig:images_2}A.

    \par The proposed controller avoids this failure mode because adaptation is driven by the filtered reference-model tracking error $\mathbf{r}$ rather than $\mathbf{r}_{\mathrm{imp}}$, keeping the torque command smooth and the impedance error lower than AWORM throughout the interaction.

    \subsubsection{Performance Under Infeasibility}
    \par Fig.~\ref{fig:images_2}C--F evaluates the proposed controller under limits of $\pm 15$\,Nm on joints $1$--$4$, a regime where the hard torque QP with $\mathbf{s}=\mathbf{0}$ is frequently infeasible during interaction even though the hard CBF acceleration box remains nonempty (Proposition~\ref{prop: Lipschitz QP}(P3)). Fig.~\ref{fig:images_2}C compares the applied wrench $\norm{\mathbf{F}_h(t)}$ against the online wrench metric $F_{\max}(\mathbf{x}(t))$ from~\eqref{eq: wrench_metric}. Whenever $\norm{\mathbf{F}_h(t)} > F_{\max}(\mathbf{x}(t))$, no acceleration can satisfy both the hard CBF bounds and the zero-slack torque constraints, and the soft-constrained QP must activate $\mathbf{s}^* > \mathbf{0}$.

    \par Under repeated torque infeasibility, impedance tracking and $\norm{\mathbf{r}(t)}$ remain bounded, torques stay smooth, and $\mathbf{x}(t)\in\mathcal{C}^*$ because the CBF constraints remain hard.

    \begin{figure*}[t]
      \centering
      \includegraphics[width=\textwidth]{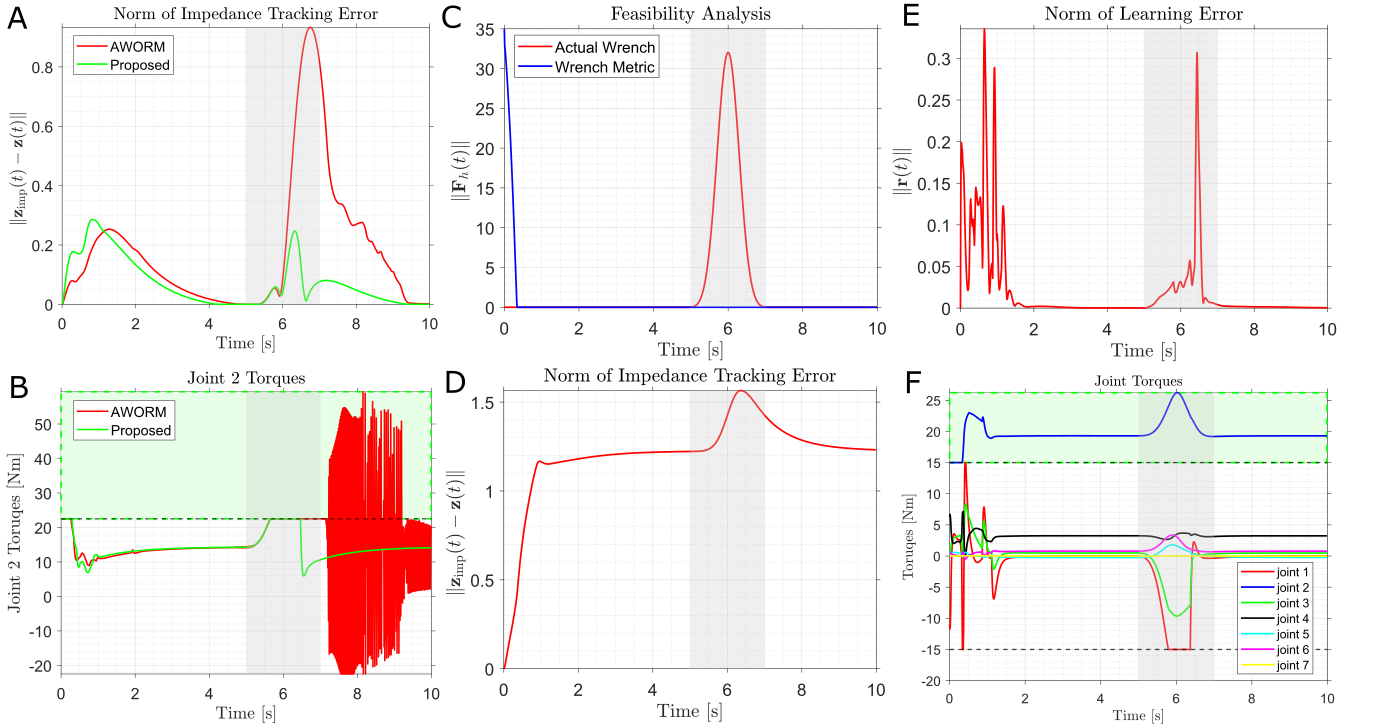}
      \caption{\textbf{A}: Impedance tracking error $\norm{\mathbf{z}_{\mathrm{imp}}(t)-\mathbf{z}(t)}$ for the proposed controller (green) and AWORM (red) with torque limits of $\pm 22.5$\,Nm on joints $1$--$4$. \textbf{B}: Joint~2 torques under the same limits; the green region indicates violations. AWORM exhibits high-frequency chattering and transient torque violations. \textbf{C}--\textbf{F}: Proposed controller with further tightened limits of $\pm 15$\,Nm on joints $1$--$4$. \textbf{C}: $\norm{\mathbf{F}_h(t)}$ versus $F_{\max}(\mathbf{x}(t))$, indicating hard-constraint infeasibility when the applied wrench exceeds the metric. \textbf{D}: Impedance tracking error $\norm{\mathbf{z}_{\mathrm{imp}}(t)-\mathbf{z}(t)}$. \textbf{E}: Reference-model learning error $\norm{\mathbf{r}(t)}$. \textbf{F}: Joint torques; the proposed controller remains smooth with no chattering.}
      \label{fig:images_2}
    \end{figure*}

%% file: section/Conclusion.tex
\section{Conclusion}\label{sec:conclusion}
    \par We presented a safety-critical online adaptive impedance controller combining a composed position-velocity NCBF, IT2-FLS adaptation with a command-driven reference model, and a soft-constrained QP with exact penalties. On $(\mathbf{x}, \mathbf{F}_h) \in \mathcal{A}$, the closed loop keeps $\mathbf{u} \in \mathcal{U}$, renders $\mathcal{C}^* \subseteq \mathcal{C}_0$ invariant, and yields UUB impedance-tracking error. Simulations on a 7-DOF KINOVA arm with $50\%$ parametric mismatch, unmodeled friction, and a pulsed interaction wrench showed improved tracking over NIC and AWORM, robustness to increasing mismatch, elimination of AWORM bursting, and smooth behaviour under tightened torque limits monitored by $F_{\max}(\mathbf{x})$.

    \par Future work will investigate neural-network function approximation, task-space safety constraints, and experimental validation on additional manipulators.

%% file: section/Appendix.tex
\appendix
\subsection{Proof of Theorem~\ref{Thm: Main Result}}\label{appendix: proof}
    \subsubsection{Statement~\ref{stat: 1}}
        By the hypotheses of Theorem~\ref{Thm: Main Result}, $(\mathbf{x}, \mathbf{F}_h) \in \mathcal{A}$, so the hard-constrained QP~\eqref{eq: Torque Constraints} (with $\mathbf{s} = \mathbf{0}$) is feasible. By the theory of exact penalty functions~\cite{kerrigan_soft_2000}, for a sufficiently large penalty weight $\rho$ exceeding the optimal Lagrange multiplier of the hard-constrained QP, the soft-constrained QP solution satisfies $\mathbf{s}^* = \mathbf{0}$. Thus, the control torque~\eqref{eq: control torque} satisfies $\mathbf{u} \in \mathcal{U}$ by construction.
    \hfill \QED

    \subsubsection{Statement~\ref{stat: 2}}
        The CBF constraints~\eqref{eq: robust_CBF} are hard in~\eqref{eq: Torque Constraints} and are never relaxed by $\mathbf{s}$. The result follows from~\cite[Thm.~1]{glotfelter_nonsmooth_2017}, Assumption~\ref{assum: valid_cbf}, and~\eqref{eq: robust_CBF}.
    \hfill \QED

    \subsubsection{Statement~\ref{stat: 3}}
        \par The proof employs a composite Lyapunov approach. We first analyse the reference-model subsystem ($V_1$) using $\mathbf{r}$, then the impedance tracking subsystem using $\mathbf{r}_{\mathrm{mm}}$.

        \par Let $\boldsymbol{\phi}_{\theta} \coloneqq \boldsymbol{\theta} - \boldsymbol{\theta}^*$ and $\boldsymbol{\psi}_{w} \coloneqq \mathbf{w} - \hat{\mathbf{w}}$. Recall that $\mathbf{e} = \mathbf{q}_{\mathrm{ref}} - \mathbf{q}$, which has dynamics $\ddot{\mathbf{e}} = \ddot{\mathbf{q}}_{\mathrm{ref}} - \ddot{\mathbf{q}}$. From \eqref{eq: Reference Governor}, $\ddot{\mathbf{q}}_{\mathrm{ref}} = \ddot{\mathbf{q}}_{\mathrm{cmd}} + \mathbf{K}_v(\dot{\mathbf{q}} - \dot{\mathbf{q}}_{\mathrm{ref}}) + \mathbf{K}_p(\mathbf{q} - \mathbf{q}_{\mathrm{ref}})$. Substituting \eqref{eq: d_def} into \eqref{eq: dcbf} gives $\mathbf{d}_{\mathrm{CBF}}(t) = -\boldsymbol{\zeta}(\boldsymbol{\xi})^\top \boldsymbol{\phi}_{\theta} - \mathbf{e}_d(t)$, thus \eqref{eq: Simplified Dynamics} can be rewritten as $\ddot{\mathbf{q}} = \ddot{\mathbf{q}}_{\mathrm{cmd}} + \ddot{\mathbf{q}}_{\mathrm{s}} - \mathbf{e}_d - \boldsymbol{\zeta}(\boldsymbol{\xi})^\top \boldsymbol{\phi}_{\theta}$. It follows that $\ddot{\mathbf{e}} = -\mathbf{K}_v \dot{\mathbf{e}} - \mathbf{K}_p \mathbf{e} + \mathbf{e}_d - \ddot{\mathbf{q}}_{\mathrm{s}} + \boldsymbol{\zeta}(\boldsymbol{\xi})^\top \boldsymbol{\phi}_{\theta}$. As such, $\dot{\mathbf{r}} = - \mathbf{K}_r \mathbf{r} + \mathbf{e}_d - \ddot{\mathbf{q}}_{\mathrm{s}} + \boldsymbol{\zeta}(\boldsymbol{\xi})^\top \boldsymbol{\phi}_{\theta}$ (follows from the definitions of $\mathbf{K}_v, \mathbf{K}_p, \mathbf{r}$).

        \par Consider the Lyapunov function candidate
        \begin{equation}\label{eq: Lyapunov candidate}
            V_1 = \tfrac{1}{2} \| \mathbf{r} \|^2 
            + \tfrac{1}{2} \boldsymbol{\phi}_{\theta}^\top \boldsymbol{\Gamma}_{\theta}^{-1} \boldsymbol{\phi}_{\theta} 
            + \tfrac{1}{2} \boldsymbol{\psi}_{w}^\top \boldsymbol{\Gamma}_{w}^{-1} \boldsymbol{\psi}_{w},
        \end{equation}
        which is positive-definite and radially unbounded (Rayleigh-Ritz Theorem). Differentiating and substituting~\eqref{eq: theta law} and \eqref{eq: w Law} with $\dot{\boldsymbol{\phi}}_\theta = \dot{\boldsymbol{\theta}}$, and $\dot{\boldsymbol{\psi}}_w = -\dot{\hat{\mathbf{w}}}$ gives:
        \begin{equation}\label{eq: Lyapunov Derivative}
            \begin{aligned}
                \dot{V}_1 &= \mathbf{r}^\top (-\mathbf{K}_r \mathbf{r} + \mathbf{e}_d - \ddot{\mathbf{q}}_{s} + \boldsymbol{\zeta}(\boldsymbol{\xi})^\top \boldsymbol{\phi}_{\theta}) \\
                & \quad - \boldsymbol{\phi}_{\theta}^\top \operatorname{Proj}_{\boldsymbol{\Gamma}_\theta}\!\big(\boldsymbol{\theta}, -\boldsymbol{\zeta}(\boldsymbol{\xi}) \mathbf{r}, f(\boldsymbol{\theta}))
                + \boldsymbol{\psi}_{w}^\top |\mathbf{r}| - \sigma_{w} \boldsymbol{\psi}_{w}^\top \hat{\mathbf{w}} \\
                &\leq \mathbf{r}^\top (-\mathbf{K}_r \mathbf{r} + \mathbf{e}_d - \ddot{\mathbf{q}}_{s} + \boldsymbol{\zeta}(\boldsymbol{\xi})^\top \boldsymbol{\phi}_{\theta}) \\
                & \quad - \boldsymbol{\phi}_{\theta}^\top \boldsymbol{\zeta}(\boldsymbol{\xi}) \mathbf{r}
                + \boldsymbol{\psi}_{w}^\top |\mathbf{r}| - \sigma_{w} \boldsymbol{\psi}_{w}^\top \hat{\mathbf{w}}.
            \end{aligned}
        \end{equation}
        The cross terms $\mathbf{r}^\top\boldsymbol{\zeta}^\top \boldsymbol{\phi}_\theta - \boldsymbol{\phi}_\theta^\top \boldsymbol{\zeta}\mathbf{r}$ cancel, and using $\mathbf{r}^\top\mathbf{e}_d \leq |\mathbf{r}|^\top\mathbf{w}$ and $|\mathbf{r}|^\top\mathbf{w} - \boldsymbol{\psi}_w^\top|\mathbf{r}| = |\mathbf{r}|^\top\hat{\mathbf{w}}$, Lemma~\ref{lem: Robust Bound} gives $0 \leq |\mathbf{r}|^\top\hat{\mathbf{w}} - \mathbf{r}^\top\ddot{\mathbf{q}}_s \leq \epsilon$ with $\epsilon = \sum_i\varepsilon_i$. Applying the Cauchy-Schwartz and Young's Inequalities to the leakage term yields $\sigma_{w} \boldsymbol{\psi}_{w}^\top \hat{\mathbf{w}} \leq -\frac{1}{2} \sigma_{w} \|\boldsymbol{\psi}_{w} \|^2 + \frac{1}{2} \sigma_{w} W^2$. Thus, we have $\dot{V}_1 \leq -\mathbf{r}^\top \mathbf{K}_r \mathbf{r} -\tfrac{1}{2} \sigma_{w} \|\boldsymbol{\psi}_{w} \|^2 + \epsilon + K$, where $K = \tfrac{1}{2} \sigma_{w} W^2$.
        
        \par Let $\mathbf{E} = (\mathbf{r}, \boldsymbol{\phi}_{\theta}, \boldsymbol{\psi}_{w}) \in \R^{2n + N}$, $z_1 = \min{\left(2\lambda_{\min}(\mathbf{K}_r),\, \sigma_w \lambda_{\min}(\boldsymbol{\Gamma}_w) \right)}$, and $V_{\theta} = \tfrac{1}{2} \boldsymbol{\phi}_{\theta}^\top \boldsymbol{\Gamma}_{\theta}^{-1} \boldsymbol{\phi}_{\theta}$. By the projection operator~\cite{lavretsky_projection_2012} and Assumption~\ref{assum: bounds}, $\norm{\boldsymbol{\phi}_{\theta}} \leq \Phi_{\max} \implies V_{\theta} \leq V_{\theta, \max}$, thus it follows that $-(\tfrac{1}{2} \| \mathbf{r} \|^2  + \tfrac{1}{2} \boldsymbol{\psi}_{w}^\top \boldsymbol{\Gamma}_{w}^{-1} \boldsymbol{\psi}_{w}) = -V_1 + V_{\theta} \leq -V_1 + V_{\theta, \max}$. Substituting this into $\dot{V}_1$ gives $\dot{V}_1 \leq -z_1 V_1 + (\epsilon + K + z_1 V_{\theta, \max})$.

        \par By the comparison lemma~\cite{khalil_nonlinear_2002}, $\mathbf{E}$, and thus $\mathbf{r}$, is UUB. Since $\boldsymbol{\Lambda}$ is positive-definite and $\dot{\mathbf{e}} = -\boldsymbol{\Lambda}\mathbf{e} + \mathbf{r}$ is a stable first-order linear time-invariant (LTI) system, the vector $(\mathbf{e},\dot{\mathbf{e}})$ is UUB.

        \par Define the impedance filtered error
        \begin{equation}\label{eq: r_imp}
            \mathbf{r}_{\mathrm{imp}} = \dot{\mathbf{e}}_{\mathrm{imp}} + \boldsymbol{\Lambda} \mathbf{e}_{\mathrm{imp}}, \quad \mathbf{e}_{\mathrm{imp}} = \mathbf{q}_{\mathrm{imp}} - \mathbf{q}.
        \end{equation}
        Now, we consider the \emph{model impedance mismatch}. Let $\mathbf{e}_{\mathrm{mm}} = \mathbf{q}_{\mathrm{imp}} - \mathbf{q}_{\mathrm{ref}}$ and $\mathbf{r}_{\mathrm{mm}} = \mathbf{r}_{\mathrm{imp}} - \mathbf{r} = \dot{\mathbf{e}}_{\mathrm{mm}} + \boldsymbol{\Lambda} \mathbf{e}_{\mathrm{mm}}$. The mismatch error has dynamics $\ddot{\mathbf{e}}_{\mathrm{mm}} = \ddot{\mathbf{q}}_{\mathrm{imp}} - \ddot{\mathbf{q}}_{\mathrm{ref}}$. Recall from Section~\ref{subsec: pch} that $\ddot{\mathbf{q}}_{\mathrm{QP}} = \ddot{\mathbf{q}}_{\mathrm{nom}} - \ddot{\mathbf{q}}_{\mathrm{cmd}}$ and $\ddot{\mathbf{q}}_{\mathrm{nom}} = \ddot{\mathbf{q}}_{\mathrm{cmd}} + \ddot{\mathbf{q}}_{\mathrm{QP}}$. Therefore, $\ddot{\mathbf{q}}_{\mathrm{cmd}} = (\ddot{\mathbf{q}}_{\mathrm{imp}} - \ddot{\mathbf{q}}_{\mathrm{QP}}) + \mathbf{K}_v \dot{\mathbf{e}}_{\mathrm{mm}} + \mathbf{K}_p \mathbf{e}_{\mathrm{mm}}$ and $\ddot{\mathbf{e}}_{\mathrm{mm}} =  -\mathbf{K}_v \dot{\mathbf{e}}_{\mathrm{mm}} - \mathbf{K}_p \mathbf{e}_{\mathrm{mm}} + \ddot{\mathbf{q}}_{\mathrm{QP}}$ and $\dot{\mathbf{r}}_{\mathrm{mm}} = -\mathbf{K}_r \mathbf{r}_{\mathrm{mm}} + \ddot{\mathbf{q}}_{\mathrm{QP}}$. Moreover, $\ddot{\mathbf{q}}_{\mathrm{cmd}}$ is bounded by the hard acceleration box constraints in~\eqref{eq: Torque Constraints}, and $\ddot{\mathbf{q}}_{\mathrm{nom}}$ is bounded by construction; hence $\ddot{\mathbf{q}}_{\mathrm{QP}}$ is bounded.

        \par Consider the Lyapunov function candidate
        \begin{equation}\label{eq: Lyapunov candidate 2}
            V_2 = \tfrac{1}{2} \| \mathbf{r}_{\mathrm{mm}} \|^2,
        \end{equation}
        which is positive-definite and radially unbounded (Rayleigh-Ritz Theorem). Differentiating $V_2$ gives $\dot{V_2} = -\mathbf{r}_{\mathrm{mm}}^\top \mathbf{K}_r \mathbf{r}_{\mathrm{mm}} + \mathbf{r}_{\mathrm{mm}}^\top \ddot{\mathbf{q}}_{\mathrm{QP}}$. By Cauchy-Schwartz, we have $\dot{V_2} \leq -\lambda_{\min}(\mathbf{K}_r) \norm{\mathbf{r}_{\mathrm{mm}}}^2 + \norm{\mathbf{r}_{\mathrm{mm}}} \norm{\ddot{\mathbf{q}}_{\mathrm{QP}}}$ and applying Young's Inequality gives $\dot{V_2} \leq -\lambda_{\min}(\mathbf{K}_r) V_2 + \tfrac{1}{2\lambda_{\min}(\mathbf{K}_r)} \norm{\ddot{\mathbf{q}}_{\mathrm{QP}}}^2$. By the comparison lemma~\cite{khalil_nonlinear_2002}, $\mathbf{r}_{\mathrm{mm}}$ is UUB. From \eqref{eq: r_imp}, $\mathbf{r}_{\mathrm{imp}} = \mathbf{r}_{\mathrm{mm}} + \mathbf{r}$, thus $\mathbf{r}_{\mathrm{imp}}$ is UUB. $\dot{\mathbf{e}}_{\mathrm{imp}} = -\boldsymbol{\Lambda}\mathbf{e}_{\mathrm{imp}} + \mathbf{r}_{\mathrm{imp}}$ is a stable first-order LTI system; thus, the vector $(\mathbf{e}_{\mathrm{imp}},\dot{\mathbf{e}}_{\mathrm{imp}})$ is UUB. In the special case where $\norm{\ddot{\mathbf{q}}_{\mathrm{QP}}} = 0$ (i.e., no QP constraints are active), $\mathbf{r}_{\mathrm{mm}}$ is asymptotically stable; thus, $\mathbf{r}_{\mathrm{imp}}$ is bounded by $\mathbf{r}$.

        \par Finally, as $(\mathbf{e}_{\mathrm{imp}}, \dot{\mathbf{e}}_{\mathrm{imp}})$ is UUB, then with Lemma~\ref{lem: Joint-Task Error}, there exists $T>0$ such that $\norm{\mathbf{z}_{\mathrm{imp}}(t) - \mathbf{z}(t)} \leq \rho(\epsilon)$ for all $t \geq T$, where $\rho(\cdot)$ is some class-$\mathcal{K}$ function.
    \hfill \QED